%% file: ms.tex
\documentclass[fleqn]{article} 
\pdfoutput=1
\usepackage{a4wide}
\usepackage{mathtools}
\usepackage{geometry}
\usepackage{graphicx}
\usepackage{amssymb}
\usepackage{amsthm}
\usepackage{amsmath}
\usepackage{dsfont}
\usepackage{color}
\usepackage{bm}
\usepackage{extarrows}
\usepackage{chngcntr}
\usepackage{enumitem}
\usepackage{multirow}
\usepackage{makecell}
\usepackage{breakcites}
\usepackage[utf8]{inputenc}
\usepackage[english]{babel}
\usepackage[hidelinks]{hyperref}
\usepackage{cleveref}
\usepackage{calc}

\input{Macros/Makros}

\input{Macros/Makros_Design}

\input{Macros/Makros_Abkuerzungen_Englisch}
\input{Macros/Makros_Theorems_Englisch}

\title{\textbf{Total Stability of SVMs and Localized SVMs}}
\date{\today}
\author{\textbf{Hannes K\"ohler}\thanks{Corresponding author; email: \href{mailto:hannes.koehler@uni-bayreuth.de}{\texttt{hannes.koehler@uni-bayreuth.de}}}~ and \textbf{Andreas Christmann}\\
	Department of Mathematics, University of Bayreuth, Germany
	}

\begin{document}
\maketitle

\input{Sections/00_Abstract.tex}

\section{Introduction}\label{Sec:Rob_Intro}

\input{Sections/01_Intro.tex}

\section{Total stability of SVMs}\label{Sec:Rob_Stab}

\input{Sections/02_Stability.tex}

\section{Total stability of localized SVMs}\label{Sec:Rob_Loc}

\input{Sections/03_Localization.tex}
\input{Sections/03a_Loc_GleicheAufteilung.tex}
\input{Sections/03b_Loc_VerschAuft.tex}

\section{Discussion}\label{Sec:Rob_Discussion}
\input{Sections/04_Discussion.tex}

\appendix

\section{Auxiliary results regarding the stability of SVMs}\label{Sec:Rob_Aux_Stability}

\input{Sections/A02_Aux_Stability.tex}

\section{Proofs}
\input{Sections/B02_Stability.tex}

\input{Sections/B03_Localization.tex}

\input{ms.bbl}
\end{document}

%% file: Macros/Makros.tex
\usepackage{xifthen}
\usepackage{xspace}

\makeatletter
\def\moverlay{\mathpalette\mov@rlay}
\def\mov@rlay#1#2{\leavevmode\vtop{%
		\baselineskip\z@skip \lineskiplimit-\maxdimen
		\ialign{\hfil$\m@th#1##$\hfil\cr#2\crcr}}}
\newcommand{\charfusion}[3][\mathord]{
	#1{\ifx#1\mathop\vphantom{#2}\fi
		\mathpalette\mov@rlay{#2\cr#3}
	}
	\ifx#1\mathop\expandafter\displaylimits\fi}
\makeatother

\definecolor{darkblue}{RGB}{0,0,204}
\definecolor{pinegreen}{RGB}{0,100,80}

\newcommand{\leadingzero}[1]{\ifnum #1<10 0\the#1\else\the#1\fi} 
\newcommand{\todayII}{\leadingzero{\day}.\leadingzero{\month}.\the\year}     
\newcommand{\einschraenkung}{\,\rule[-2mm]{0.1mm}{5mm}\,{}}

\newcommand{\enm}[1]{\ensuremath{#1}\xspace}

\newcommand{\limn}{\enm{\lim_{n\to\infty}}}
\newcommand{\R}{\enm{\mathbb{R}}}

\newcommand{\N}{\enm{\mathbb{N}}}
\newcommand{\eps}{\enm{\varepsilon}}
\newcommand{\lb}{\enm{\lambda}}
\newcommand{\LandauO}{\enm{\mathcal{O}}}
\newcommand{\Ind}[1][]{\enm{\mathds{1}_{#1}}}
\newcommand{\bigcupdot}{\charfusion[\mathop]{\bigcup}{\cdot}}

\newcommand{\norm}[2]{\enm{\left|\left|#2\right|\right|_{#1}}}
\newcommand{\normSup}[1]{\norm{\infty}{#1}}
\newcommand{\normH}[1]{\norm{\H}{#1}}
\newcommand{\normArb}[1]{\norm{\bullet}{#1}}
\newcommand{\normLppx}[1]{\norm{\Lppx}{#1}} 
\newcommand{\normLppxb}[1]{\norm{\Lppxb}{#1}} 
\newcommand{\normLppix}[2][i]{\norm{\Lppix[#1]}{#2}} 
\newcommand{\normLppixb}[1]{\norm{\Lppixb}{#1}} 
\newcommand{\normLppibxb}[1]{\norm{\Lppibxb}{#1}} 
\newcommand{\normLppibdachx}[1]{\norm{\Lppibdachx}{#1}} 
\newcommand{\normLapix}[2][i]{\norm{\Lapix[#1]}{#2}} 
\newcommand{\normLapxb}[1]{\norm{\Lapxb}{#1}} 
\newcommand{\normLapibstarxb}[2][i]{\norm{\Lapibstarxb[#1]}{#2}} 
\newcommand{\normLapiaxb}[2][i]{\norm{\Lapiaxb[#1]}{#2}} 

\renewcommand{\P}{\enm{\textnormal{P}}}
\newcommand{\Px}{\enm{\P^X}}
\newcommand{\Pbed}[2][\cdot]{\enm{\P(#1|#2)}} 
\newcommand{\Q}{\enm{\textnormal{Q}}}
\newcommand{\Qx}{\enm{\Q^X}}

\newcommand{\MXY}{\enm{\mathcal{M}_1(\XY)}}

\newcommand{\BXY}{\enm{\mathcal{B}_{\XY}}}
\newcommand{\BX}{\enm{\mathcal{B}_{\X}}}
\newcommand{\BY}{\enm{\mathcal{B}_{\Y}}}
\newcommand{\X}{\enm{\mathcal{X}}}
\newcommand{\Y}{\enm{\mathcal{Y}}}
\newcommand{\XX}{\enm{\X\times\X}}
\newcommand{\XY}{\enm{\X\times\Y}}
\newcommand{\XYR}{\enm{\X\times\Y\times\R}}

\newcommand{\A}{\enm{\mathcal{A}}}
\newcommand{\ew}[2][]{\enm{\mathbb{E}_{#1}\left[#2\right]}}

\newcommand{\unif}[1][\enm{0,1}]{\enm{\mathcal{U}(#1)}}

\newcommand{\normal}[1][\enm{0,1}]{\enm{\mathcal{N}(#1)}}

\renewcommand{\L}[2][]{\enm{L_{#2}\ifthenelse{\isempty{#1}}{}{(#1)}}}
\newcommand{\Lp}{\L{p}}

\newcommand{\Lppx}{\enm{\Lp(\Px)}} 
\newcommand{\Lppxb}{\enm{\Lp(\Px\otimes\Px)}} 
\newcommand{\La}{\L{1}}
\newcommand{\Lapxb}{\enm{\La(\Px\otimes\Px)}} 

\newcommand{\Lppix}[1][i]{\enm{\Lp(\Pix[#1])}} 
\newcommand{\Lppixb}[1][i]{\enm{\Lp(\Pix[#1]\otimes\Pix[#1])}} 
\newcommand{\Lppibxb}[1][i]{\enm{\Lp(\Pibx[#1]\otimes\Pibx[#1])}} 
\newcommand{\Lppibdachx}[1][i]{\enm{\Lp(\Pibdachx[#1])}} 
\newcommand{\Lapix}[1][i]{\enm{\La(\Pix[#1])}} 
\newcommand{\Lapiaxb}[1][i]{\enm{\La(\Piax[#1]\otimes\Piax[#1])}} 
\newcommand{\Lapibstarxb}[1][i]{\enm{\La(\Pibstarx[#1]\otimes\Pibstarx[#1])}} 

\newcommand{\loss}{\enm{L}}
\newcommand{\lossshift}{\enm{\loss^\star}}
\newcommand{\risk}[1][\loss,\P]{\enm{\mathcal{R}_{#1}}} 
\newcommand{\riskshift}[1][\lossshift,\P]{\enm{\mathcal{R}_{#1}}} 
\renewcommand{\H}{\enm{H}} 
\renewcommand{\k}{\enm{k}}

\newcommand{\fa}{\enm{f_1}}
\newcommand{\fb}{\enm{f_2}}
\newcommand{\fj}{\enm{f_i}}
\newcommand{\ftilde}[1][]{\ifthenelse{\isempty{#1}}{\enm{\tilde{f}}}{\enm{\tilde{f}_{#1}}}}

\newcommand{\ftheok}{\enm{f_{\loss,\P,\lb,\k}}} 
\newcommand{\fshifttheok}{\enm{f_{\lossshift,\P,\lb,\k}}} 
\newcommand{\frob}{\enm{f_{\P,\lb,\k}}} 
\newcommand{\froba}{\enm{f_{\P_1,\lb_1,\k_1}}} 
\newcommand{\frobb}{\enm{f_{\P_2,\lb_2,\k_2}}} 

\newcommand{\aib}[1][i]{\enm{a(#1,b)}}
\newcommand{\distReg}[2][b]{\enm{\textnormal{d}_{#2,#1}}} 
\newcommand{\XBbold}{\enm{\bm{\mathcal{X}_B}}}
\newcommand{\XiAibold}[1][i]{\enm{\bm{\mathcal{X}_{A_{#1}}^{(#1)}}}}
\newcommand{\XBstarbold}{\enm{\bm{\mathcal{X}_B^{*}}}}
\newcommand{\Xb}[1][b]{\enm{\X_{#1}}}
\newcommand{\Xbstar}[1][b]{\enm{\mathcal{X}_{#1}^*}}
\newcommand{\Xia}[1][i]{\enm{\X_{a}^{(#1)}}}
\newcommand{\Xiai}[1][i]{\enm{\X_{a_{#1}}^{(#1)}}}
\newcommand{\Xiaib}[1][i]{\enm{\X_{\aib[#1]}^{(#1)}}}
\newcommand{\Xiaflex}[2]{\enm{\X_{#2}^{(#1)}}}
\newcommand{\Pix}[1][i]{\enm{\P_{#1}^X}}
\newcommand{\Pib}[1][i]{\enm{\P_{#1,b}}}
\newcommand{\Pibx}[1][i]{\enm{\P_{#1,b}^{X}}}
\newcommand{\Pibdach}[1][i]{\enm{\hat{\P}_{#1,b}}}
\newcommand{\Pibdachx}[1][i]{\enm{\hat{\P}_{#1,b}^{X}}}
\newcommand{\Pia}[1][i]{\enm{\P_{#1,a}}}
\newcommand{\Piax}[1][i]{\enm{\P_{#1,a}^X}}
\newcommand{\Pibstar}[1][i]{\enm{\P_{#1,b}^*}}
\newcommand{\Pibstarx}[1][i]{\enm{(\P_{#1,b}^*)^X}}
\newcommand{\Pjaibx}[2][i]{\enm{\P_{#2,\Xiaib[#1]}^X}} 
\newcommand{\Pjai}[2][i]{\enm{\P_{#2,\Xia[#1]}}} 
\newcommand{\Pjaix}[2][i]{\enm{\P_{#2,\Xia[#1]}^X}} 
\newcommand{\lbbold}{\enm{\bm{\lambda}}} 
\newcommand{\lbibold}[1][i]{\enm{\bm{\lambda_{#1}}}} 
\newcommand{\lbijstarbold}[1][i,j]{\enm{\bm{\lambda_{#1}^{*}}}} 
\newcommand{\lbib}[1][i]{\enm{\lb_{#1,b}}}
\newcommand{\lbia}[1][i]{\enm{\lb_{#1,a}}}
\newcommand{\lbijbstar}[1][i,j,b]{\enm{\lb_{#1}^*}}
\newcommand{\lbiaib}[1][i]{\enm{\lb_{#1,\aib[#1]}}}
\newcommand{\kbold}{\enm{\bm{k}}} 
\newcommand{\kibold}[1][i]{\enm{\bm{k_{#1}}}} 
\newcommand{\kistarbold}[1][i]{\enm{\bm{k_{#1}^{*}}}} 
\newcommand{\ktilde}{\enm{\tilde{\k}}}
\newcommand{\kib}[1][i]{\enm{\k_{#1,b}}}
\newcommand{\kibdach}[1][i]{\enm{\hat{\k}_{#1,b}}}
\newcommand{\kia}[1][i]{\enm{\k_{#1,a}}}
\newcommand{\kibstar}[1][i,b]{\enm{\k_{#1}^*}}
\newcommand{\kiaib}[1][i]{\enm{\k_{#1,\aib[#1]}}}
\newcommand{\Htilde}{\enm{\tilde{\H}}}
\newcommand{\Hib}[1][i]{\enm{\H_{#1,b}}}
\newcommand{\Hibdach}[1][i]{\enm{\hat{\H}_{#1,b}}}
\newcommand{\Hia}[1][i]{\enm{\H_{#1,a}}}

\newcommand{\fiext}[1][i]{\enm{f_{\P_{#1},\lbibold[#1],\kibold[#1]}}} 
\newcommand{\fiextreg}[1][i]{\enm{f_{\P_{#1},\lbibold[#1],\kibold[#1],\XiAibold[#1]}}} 
\newcommand{\fib}[1][i]{\enm{f_{#1,b}}} 
\newcommand{\fibext}[1][i]{\enm{f_{\P_{#1,b},\lb_{#1,b},\k_{#1,b}}}} 
\newcommand{\fibdach}[1][i]{\enm{\hat{f}_{#1,b}}} 
\newcommand{\fibextdach}[1][i]{\enm{\hat{f}_{\P_{#1,b},\lb_{#1,b},\k_{#1,b}}}} 
\newcommand{\fibdachext}[1][i]{\enm{f_{\hat{\P}_{#1,b},\lb_{#1,b},\hat{\k}_{#1,b}}}} 
\newcommand{\fiaext}[1][i]{\enm{f_{\Pia[#1],\lbia[#1],\kia[#1]}}} 
\newcommand{\fiaextdach}[1][i]{\enm{\hat{f}_{\Pia[#1],\lbia[#1],\kia[#1]}}} 
\newcommand{\fastarproofLp}{\enm{f_{\P_{1},\lbijstarbold[1,1],\kistarbold[1],\XBstarbold}}}
\newcommand{\fbstarproofLp}{\enm{f_{\P_{1},\lbijstarbold[2,1],\kistarbold[2],\XBstarbold}}}
\newcommand{\fabstartildeproofLp}{\enm{\tilde{f}_{\Pibstar[1],\lbijbstar[1,1,b],\kibstar[1,b]}}}
\newcommand{\fabstarhatproofLp}{\enm{\hat{f}_{\Pibstar[1],\lbijbstar[1,1,b],\kibstar[1,b]}}}
\newcommand{\fbbstarhatproofLp}{\enm{\hat{f}_{\Pibstar[1],\lbijbstar[2,1,b],\kibstar[2,b]}}}


%% file: Macros/Makros_Design.tex
\usepackage{natbib}

\usepackage{etoolbox}
\makeatletter
\newcommand\bibstyle@comma{\bibpunct(),a,,}
\newcommand\bibstyle@semicolon{\bibpunct();a,,}
\makeatother
\pretocmd\citet{\citestyle{comma}}\relax\relax
\pretocmd\Citet{\citestyle{comma}}\relax\relax
\pretocmd\citep{\citestyle{semicolon}}\relax\relax
\pretocmd\Citep{\citestyle{semicolon}}\relax\relax

\makeatletter
\def\cleardoublepage{\clearpage\if@twoside \ifodd\c@page\else
	\hbox{}
	\vspace*{\fill}
	\thispagestyle{empty}
	\newpage
	\if@twocolumn\hbox{}\newpage\fi\fi\fi}
\makeatother

%% file: Macros/Makros_Abkuerzungen_Englisch.tex
\newcommand{\wenn}{\text{, if }}
\newcommand{\sonst}{\text{, else}}

\newcommand{\vgl}{\enm{\text{cf.\@}}}

\newcommand{\seite}{\enm{\text{p.\@}}}

\newcommand{\eg}{\enm{\text{e.g.\@}}}
\newcommand{\ie}{\enm{\text{i.e.\@}}}

%% file: Macros/Makros_Theorems_Englisch.tex
\newtheorem{mythm}{Theorem}
\numberwithin{mythm}{section}
\newtheorem{mylem}[mythm]{Lemma}

\theoremstyle{definition}
\newtheorem{mydef}[mythm]{Definition}

\theoremstyle{remark}
\newtheorem{mybem}[mythm]{Remark}

\numberwithin{figure}{section}

\crefname{mythm}{theorem}{theorems}
\crefname{mylem}{lemma}{lemmas}

%% file: Sections/00_Abstract.tex
\begin{abstract}

	Regularized kernel-based methods such as support vector machines (SVMs) typically depend on the underlying probability measure $\textnormal{P}$ (respectively an empirical measure $\textnormal{D}_n$ in applications) as well as on the regularization parameter $\lambda$ and the kernel $k$. Whereas classical statistical robustness only considers the effect of small perturbations in $\textnormal{P}$, the present paper investigates the influence of simultaneous slight variations in the whole triple $(\textnormal{P},\lambda,k)$, respectively $(\textnormal{D}_n,\lambda_n,k)$, on the resulting predictor. Existing results from the literature are considerably generalized and improved. In order to also make them applicable to big data, where regular SVMs suffer from their super-linear computational requirements, we show how our results can be transferred to the context of localized learning. Here, the effect of slight variations in the applied regionalization, which might for example stem from changes in $\textnormal{P}$ respectively $\textnormal{D}_n$, is considered as well.
\end{abstract}

%% file: Sections/01_Intro.tex
Let \XY be a set and let $\P$ be the distribution of a pair of random variables $(X,Y)$ with values in \XY where $X$ is the input variable and $Y$ is the real-valued output variable. The goal of statistical machine learning is to find a function $f: \X\to\Y$ which relates $X$ to $Y$, \ie which can be used to predict the unknown output variable based on a given input variable, with (almost) no prior knowledge about \P. One way to approach such prediction problems, both for regression and classification purposes, is employing \it support vector machines \rm (SVMs) which perform regularized empirical risk minimization on special Hilbert spaces of functions, so-called \it reproducing kernel Hilbert spaces \rm (RKHSs), and which have been the focus of extensive theoretical investigations \citep[\vgl][among others]{vapnik1995,vapnik1998,schoelkopf2002,cucker2007,steinwart2008}. 

To be more specific, an SVM is based on a \it loss function \rm $\loss: \XYR\to [0,\infty)$ which quantifies the quality of a prediction $f(x)$ by $\loss(x,y,f(x))$ if the observed output variable belonging to $x$ is $y$. The loss function specifies the exact goal of the prediction. For example, typical loss functions for classification tasks include the \it hinge loss\rm, the \it least squares loss \rm and the \it logistic loss\rm. For regression tasks, the \it least squares loss \rm is often used to estimate the conditional mean function, whereas the \it pinball loss \rm is suited to quantile regression. 

Based on the loss function of choice, one can define the \it \loss-risk \rm (or just \it risk\rm) \risk as the expectation of said loss function, \ie
\begin{align*}
\risk(f) := \ew[\P]{\loss(X,Y,f(X))}.
\end{align*}
Since the loss measures the quality of a specific prediction $f(x)$, the risk quantifies the quality of the whole predictor $f$ and we aim at finding a predictor whose risk is small. However, because the true underlying distribution \P is unknown in machine learning problems, it is impossible to minimize \risk directly. Instead, one uses an available data set consisting of observations from \P to calculate the \it empirical risk \rm as an approximation of the theoretical risk. Since minimizing the empirical risk almost certainly leads to some extent of overfitting, a regularization term has to be added, which also makes the resulting minimization problem well-posed in Hadamard's sense \citep[\vgl][]{hable2011}. An SVM \ftheok is then defined as the solution of the minimization problem
\begin{align}\label{eq:Rob_DefSVM}
\ftheok := \arg\inf_{f\in\H} \risk(f) + \lb \normH{f}^2.
\end{align}
Here, $\lb>0$ is a regularization parameter which controls the amount of regularization and \H is the RKHS of a measurable \it kernel on \X\rm, \ie a symmetric and positive semidefinite function $\k: \XX\to\R$ \citep[\vgl][]{aronszajn1950,schoelkopf2002,berlinet2004,cucker2007}. We will often be interested in \it bounded kernels \rm for which we define $\normSup{\k} := \sup_{x\in\X} \sqrt{k(x,x)}$. Furthermore, we need the so-called \it canonical feature map \rm $\Phi:X\to\H$ defined by $\Phi(x):=\k(\cdot,x)$ for $x\in\X$. This canonical feature map satisfies the \it reproducing property \rm 
\begin{align}\label{eq:Rob_ReprodProperty}
	\langle f,\Phi(x) \rangle_\H = f(x) \qquad \forall\, x\in\X, f\in\H\,,
\end{align}
from which one can easily deduce
\begin{align}\label{eq:Rob_FeatureMap}
	\langle \Phi(x_1),\Phi(x_2)\rangle_{\H} = k(x_1,x_2) \qquad \forall\, x_1,x_2\in\X\,,
\end{align}
\vgl \citet[Definition~2.9]{schoelkopf2002}. Lastly, we will usually assume \X to be a complete and separable metric space equipped with the Borel $\sigma$-algebra \BX and \Y to be a closed subset of \R equipped with \BY (for brevity of notation, we will from now on often omit explicitly stating the $\sigma$-algebras and instead always assume sets to be equipped with their respective Borel $\sigma$-algebra when not explicitly stated otherwise), and therefore also assume that $\P\in\MXY$ with \MXY denoting the set of all Borel probability measures on \XY. 

It can be shown that suitable choices of kernel, loss function and regularization parameter (with the last one depending on the size of the given data set) lead to desirable properties of SVMs under rather mild assumptions. These include existence, uniqueness and universal consistency as well as specific learning rates, with the last one typically requiring some more conditions on \P than the former properties for which (almost) no such conditions are needed. In addition to the books mentioned at the beginning of this article, which all include extensive introductions to SVMs as well as many results on the aforementioned properties, some more specific results on learning rates can, for example, be found in \citet{caponnetto2007,smale2007,xiang2009,steinwart2009,eberts2011,eberts2013,farooq2019}. We refer to \citet{christmann2012,christmann2016} for results on SVMs for additive models and to \citet{christmann2016b,gensler2020} for results on kernel-based pairwise learning.

In this article, we will mainly concern ourselves with the stability of SVMs, \ie how much influence simultaneous slight changes in the probability measure \P (or slight changes in the data set in the empirical case), the regularization parameter \lb and the kernel \k have on the resulting SVM. Because of this special interest in the effect of varying $(\P,\lb,\k)$, we will usually just write \frob instead of \ftheok to shorten the notation whenever \loss is clear from the context or does not need to be specified. Since \loss has to be chosen by the user depending on the problem at hand (for example, least squares loss for regression or pinball loss for quantile regression), deviations stemming from changes in the loss function are indeed desired and we are not interested in bounding them---in contrast to deviations produced by slight changes in \lb or \k which can, for example, stem from slight changes in the underlying data since \lb and the hyperparameter(s) of \k are often chosen in a data-dependent way.

There are several already existing results on different notions of classical statistical robustness of SVMs, that is, on the influence of small changes in \P (or in the data set) on the predictor \citep[\vgl][among others]{bousquet2002,christmann2004,christmann2007,christmann2008a,hable2011}. We will, however, base many of our considerations on \citet{christmann2018}, where the authors investigated the effect of simultaneous slight changes in the whole triple $(\P,\lb,\k)$ instead of only in \P and obtained results like
\begin{align}\label{eq:Rob_LandauFormulation}
\normSup{\froba-\frobb} = \LandauO\left(\norm{tv}{\P_1-\P_2}\right) + \LandauO\left(|\lb_1-\lb_2|\right) + \LandauO\left(\normSup{\k_1-\k_2}\right)
\end{align}
with known constants, and with with $\norm{tv}{\nu}$ denoting the norm of total variation of a signed measure $\nu$. 

We will first modify one such result by \citet[Theorem~2.7]{christmann2018} slightly in \Cref{Sec:Rob_Stab} in order to generalize it and make it applicable to a larger class of loss functions and to arbitrary positive regularization parameters. We will additionally derive an analogous statement concerning the \Lppix-norm, $i=1,2$, $p\in[1,\infty)$, of $\froba-\frobb$ instead of its supremum norm, with \Qx denoting the marginal distribution on \X associated with \Q, for all probability measures \Q on \XY. Afterwards, we will investigate \it localized SVMs \rm in \Cref{Sec:Rob_Loc}. 

The principal idea behind localized SVMs is to not calculate one SVM on the whole input space \X but instead split \X into different (not necessarily disjoint) subsets, calculate SVMs on these subsets and then join them together (combined with some weight functions in the case of overlapping subsets) in order to obtain a global predictor. There are three main advantages to this approach: Firstly, the calculation of SVMs is known to have super-linear (in the number of training samples) computational requirements (in time as well as in storage space), \vgl \citet{platt1998,joachims1998} among others. For large data sets, it is therefore faster and more computationally feasible to calculate several small SVMs instead of a single large one. Secondly, localization allows the algorithm to deal with different structures in different regions of the input space in different ways. For global learning approaches, it can be difficult to accurately predict a function whose complexity and volatility vary among different areas of the input space because the complexity of a predictor is usually controlled globally by some hyperparameters. For the very same reason, large differences between the conditional distributions $\Pbed[Y]{X=x}$ in different areas of \X can cause similar difficulties. A good regionalization can separate such different areas and can therefore lead to better predictions based on a given training data set. Thirdly, the use of a bounded and continuous kernel \k, which is popular in practice (with \eg the Gaussian RBF kernel satisfying both properties) and yields some useful theoretical properties, leads to all functions from the RKHS \H, and thus also the SVM, being continuous and bounded as well \citep[\vgl][Lemma~4.28]{steinwart2008}. Hence, it can be difficult for such SVMs to accurately model discontinuities in the true function, and large oscillations and overshooting can occur near these discontinuities, similar to the well-known Gibbs phenomenon occurring for Fourier series \citep[\vgl][and the references cited therein]{hewitt1979}. By dividing the input space into separate regions at these discontinuities, a good regionalization can eliminate this hindrance.

Hence, localized approaches---not only for SVMs but also for other machine learning methods which often face similar challenges---have been of interest for many years. Early theoretical investigations can be found in \citet{bottou1992,vapnik1993}, and nowadays various approaches for splitting the data into local subsets before applying SVMs exist. For example, \citet{bennett1998,wu1999,tibshirani2007} as well as \citet{chang2010} all employ decision trees. The methods proposed in these articles however differ in how the tree is generated: In the former three, an SVM is also used for each decision in the tree, whereas \citet{chang2010} split the data in an axis-parallel way and only apply SVMs to the final regions. That is, the goal of the former articles is only to improve the accuracy of the predictor whereas the latter one is also concerned with reducing training time. Another popular approach is to combine SVMs with $k$-nearest neighbor ($k$NN) methods. This has, for example, been done by \citet{zhang2006,hable2013} and by \citet{blanzieri2007,blanzieri2008,segata2010}, with the former two measuring distances (for selecting the $k$ nearest neighbors) in the input space \X and the rest measuring them in the feature space \H. \citet{segata2010} constitutes a special case among these, since the authors modified the procedure somewhat in order to speed up the process of classifying new data points: $k$NN methods usually suffer from a computationally intensive and slow prediction phase due to their construction. That is, they have to calculate a new SVM on the $k$-neighborhood of each test point whose output they have to predict. Even though these SVMs are not based on the whole training set and thus relatively faster to train, this still significantly slows down the prediction step if many predictions have to be made. To circumvent this problem, \citet{segata2010} proposed to train an SVM on the $k$-neighborhood of each point from the training set during the training phase and then use the SVM belonging to the closest training point when having to predict the output belonging to a test point. Slightly different approaches have been proposed by \citet{cheng2007,cheng2010,gu2013}, where the training data is being split into clusters by some variants of $k$-means and then an SVM is trained on each cluster. Similarly, \citet{rida1999} combined SVMs with density-based clustering for the case of having a multimodal input. 
\citet{zakai2009} showed that any consistent learning method has to be localizable because it has to behave in a local manner in order to be consistent. However, they also only investigated a special localization technique where only training samples within a ball of some fixed radius around a test point are considered when the associated output has to be predicted.

Additionally, there are several articles that provide theoretical results about localized SVMs \citep[whereas many of the aforementioned articles focused on experimental analyses, notable exceptions being][among others]{zakai2009,gu2013,hable2013} and that do often not demand special localization techniques but instead only require the resulting regionalization to satisfy some (rather mild) conditions. These include \citet{meister2016,thomann2017,blaschzyk2020} where the Gaussian RBF kernel in combination with the hinge or the least squares loss function is used to derive learning rates for such localized SVMs under some assumptions regarding the underlying distribution \P \citep[which are always needed in order to obtain learning rates because of the no-free-lunch theorem, \vgl][]{devroye1982}. Lastly, \citet{dumpert2018,dumpert2020} allow even more general regionalizations as well as more general kernels and loss functions, based on which they show localized SVMs' consistency as well as their statistical robustness with respect to the maxbias and the influence function without any restrictive assumptions about \P.

Other approaches for reducing the computational requirements of SVMs (or similar methods) include, for example, distributed learning \citep[see][among others]{christmann2007b,zhang2015,lin2017,guo2017,lin2020} and online learning \citep[see][among others]{ying2006,smale2006,guo2017b}.

We will, however, focus on a localized approach and proceed similarly to \citet{dumpert2018,dumpert2020} in \Cref{Sec:Rob_Loc}, \ie impose only very mild assumptions on regionalization, kernel, loss function and the underlying distribution, and then transfer the stability results from \Cref{Sec:Rob_Stab} to our localized SVMs. Notably, we will not make any assumptions regarding the heaviness of the tails of the conditional distributions $\Pbed[Y]{X=x}$, $x\in\X$, and especially not require \Y to be bounded (which is, for example, needed in the aforementioned results on learning rates). Lastly, we will look at the effect of not only probability measure, regularization parameters and kernels but also the regionalization varying. That is, we will investigate the stability of localized SVMs with respect to slight changes in the whole quadruple consisting of probability measure, regularization parameters, kernels and regionalization.

%% file: Sections/02_Stability.tex
In this section, we will show stability of SVMs with respect to slight changes in the triple $(\P,\lb,\k)$ consisting of probability measure, regularization parameter and kernel. Our notion of stability will be similar to that of \eqref{eq:Rob_LandauFormulation}, with the slight difference that we additionally need to consider $\sqrt{\normSup{\k_1-\k_2}}$ as an exchange for the result considerably generalizing the referenced theorem by \citet{christmann2018}, that is, the result being applicable to arbitrary positive regularization parameters and a larger class of loss functions. Thus, it will be of the type
\begin{align}\label{eq:Rob_LandauFormulation2}
	\normSup{\froba-\frobb} =&\ \LandauO\left(\norm{tv}{\P_1-\P_2}\right) + \LandauO\left(|\lb_1-\lb_2|\right)\notag\\ 
	&\ + \LandauO\left(\normSup{\k_1-\k_2}\right) + \LandauO\left(\sqrt{\normSup{\k_1-\k_2}}\right)\,.
\end{align}
Afterwards, we will derive a similar stability result which bounds the \Lppix-norm, $i=1,2$, of $\froba-\frobb$. This will be of the type
\begin{align}\label{eq:Rob_LandauFormulationLp}
	\normLppix{\froba-\frobb} =&\ \LandauO\left(\norm{tv}{\P_1-\P_2}\right) + \LandauO\left(|\lb_1-\lb_2|\right)\notag\\ 
	&\ + \LandauO\left(\normLppixb{\k_1-\k_2}\right) + \LandauO\left(\sqrt{\normLppixb{\k_1-\k_2}}\right)\,
\end{align}
and will require the same mild conditions as the one concerning the supremum norm.

As mentioned in the introduction, $\norm{tv}{\nu}$ denotes the norm of total variation of a signed measure $\nu$ on a Banach space $E$ (in our case \XY), that is
\begin{align*}
	\norm{tv}{\nu} := |\nu|(E) := \sup \left\{\left.\sum_{i=1}^{n} |\nu|(E_i) \,\right|\, E_1,\dots,E_n \text{ is a partition of } E \right\}\,.
\end{align*}
Furthermore note that $k_1-k_2$ generally is no kernel and therefore $\normSup{k_1-k_2}$ in \eqref{eq:Rob_LandauFormulation2} denotes the general supremum norm of a function instead of the special definition of $\normSup{\cdot}$ for kernels stated before \citep[which coincides with the square root of the general definition when applied to kernels, \vgl][\seite 22]{cucker2007}. 

Returning to the reduced conditions on the loss function \loss compared to the referenced theorem by \citet{christmann2018}, we only require \loss to be convex as well as Lipschitz continuous. Here, convexity refers to convexity in the last argument, \ie we call a loss function $\loss:\XYR\to[0,\infty)$ \it convex \rm if $\loss(x,y,\cdot):\R\to[0,\infty)$ is convex for all $(x,y)\in\XY$. Furthermore, we call \loss \it Lipschitz continuous \rm if there exists a constant $c\ge 0$ such that
\begin{align*}
	|L(x,y,t_1)-L(x,y,t_2)| \le c \cdot |t_1-t_2| \qquad \forall\, x\in\X,y\in\Y,t_1,t_2\in\R\,.
\end{align*}
We denote the smallest such constant by $|\loss|_1$ and call it \it Lipschitz constant \rm of \loss.

Additionally, we require the kernels used in the definition of the SVMs to be bounded which is, for example, satisfied by the popular Gaussian kernel. We will always denote the RKHS and the canonical feature map associated with a kernel by providing them with the same indices or other additional notation, for example, $\tilde{\H}_1$ and $\tilde{\Phi}_1$ denote RKHS and canonical feature map belonging to $\tilde{\k}_1$.

With these conditions on the loss function \loss and the kernels $\k_1$ and $\k_2$, it would now be possible to state a stability result of the type \eqref{eq:Rob_LandauFormulation2}. This would however require us to impose an additional condition on the probability measures $\P_1$ and $\P_2$ in order to guarantee both of the SVMs which are to be compared to uniquely exist. More specifically, it would be necessary that the RKHS $\H_i$ contains at least one function $f$ with finite risk with respect to $\P_i$ \citep[\vgl][Lemma 5.1 and Theorem 5.2]{steinwart2008}. 
One way to ensure this is to impose the moment condition $\ew[\P_i]{|Y|} < \infty$ on $\P_i$ \citep[\vgl][]{christmann2009}. Alas, this moment condition excludes heavy-tailed distributions such as the Cauchy distribution. In order to circumvent this problem, we will use so-called \it shifted loss functions\rm. These special losses have been applied in robust statistics for a long time \citep[\vgl][]{huber1967} and have been introduced to SVMs by \citet{christmann2009}. The concept of shifted loss functions is simple enough: As the name suggests, they just shift a loss function by some fixed amount. More specifically, given a loss function $\loss:\XYR\to[0,\infty)$, the \it shifted loss function \rm \lossshift is defined by
\begin{align*}
	\lossshift:&\ \XYR \to \R\,,\\
	&\ (x,y,t)\mapsto \loss(x,y,t)-\loss(x,y,0)\,.
\end{align*}
Risks and SVMs can be defined in the same way as for normal loss functions, \ie
\begin{align*}
	\riskshift(f) := \ew[\P]{\lossshift(X,Y,f(X))}
\end{align*}
and
\begin{align}\label{eq:Rob_DefShiftSVM}
	\fshifttheok := \arg\inf_{f\in\H} \riskshift(f) + \lb \normH{f}^2\,.
\end{align}
We will again just write \frob instead of \fshifttheok whenever \lossshift is clear from the context or does not need to be specified.
It is easy to see that \lossshift is convex respectively Lipschitz continuous if and only if \loss is convex respectively Lipschitz continuous and that they both have the same Lipschitz constant, \ie $|\lossshift|_1=|\loss|_1$ \citep[\vgl][Proposition 2]{christmann2009}. Hence, the properties of \loss and of \lossshift can be used interchangeably.

\citet{christmann2009} additionally showed that $\fshifttheok=\ftheok$ whenever $\risk(0)<\infty$, \ie that using shifted loss functions leads to the same results as using normal loss functions and is therefore justified, and that the use of \lossshift eliminates the need for the moment condition whenever \loss is Lipschitz continuous and the kernel is bounded. Since we required these two properties anyway, using \lossshift instead of \loss rids us of the moment condition without imposing any additional conditions. When writing \frob, we will usually refer to \fshifttheok instead of \ftheok because of these advantages of \lossshift. However, since the two functions coincide whenever both exist, we could obviously also use \ftheok in these cases.

We will now state our first main result about the stability of SVMs. As mentioned in the introduction, this is a generalization of a result (Theorem 2.7) by \citet{christmann2018}: First of all, we eliminated an additional condition on \loss that was required by \citet{christmann2018}. Previously, \loss did not only need to be convex and Lipschitz continuous but also differentiable. Since many loss functions are not differentiable (\eg pinball loss, $\eps$-insensitive loss, hinge loss), this change makes the result applicable to a considerably larger class of learning tasks. Secondly, in \citet[Theorem 2.7]{christmann2018} it was assumed that the regularization parameters $\lb_1$ and $\lb_2$ were greater than some specified positive constant, which is unsatisfactory because the regularization parameter used by an SVM has to converge to zero as the size of the training data set tends to infinity in order to achieve consistency \citep[\vgl][Theorem 8]{christmann2009}. In order to circumvent this problem, \citet{christmann2018} additionally provided another result (Theorem 2.10) in which $\lb_1$ and $\lb_2$ are allowed to be arbitrarily close to zero and instead of $\normSup{\froba-\frobb}$, as in \eqref{eq:Rob_LandauFormulation}, they bound $\norm{\H_1}{\froba-\frobb}$. Since
\begin{align}\label{eq:Rob_InfNormHNorm}
	\normSup{f} \le \normSup{\k} \cdot \normH{f}
\end{align}
for every RKHS \H and $f\in\H$ (\vgl \citeauthor{cucker2007}, \citeyear{cucker2007}, Theorem 2.9; \citeauthor{steinwart2008}, \citeyear{steinwart2008}, Lemma 4.23) and $\k_1$ is assumed to be bounded, this also translates to a bound for $\normSup{\froba-\frobb}$. Alas, this result obviously requires the RKHSs $\H_1$ and $\H_2$ be nested, $\H_2\subseteq \H_1$, and additionally uses $\norm{\H_1}{\k_1-\k_2}$ instead of the more easily interpretable $\normSup{\k_1-\k_2}$ in the bound.

In the subsequent theorem, we neither need $\lb_1$ and $\lb_2$ to be greater than some positive constant nor $\H_1$ and $\H_2$ to be nested:

\begin{mythm}\label{Thm:Rob_Stability}
	Let \X be a complete and separable metric space and $\Y\subseteq\R$ be closed. Let $\P_1,\P_2\in\MXY$ be probability measures, $\lb_1,\lb_2>0$ and $\k_1,\k_2$ be measurable and bounded kernels on \X with separable RKHSs $\H_1,\H_2$. Denote $\kappa:=\max\{\normSup{\k_1},\normSup{\k_2}\}$ and $\tau:=\min\{\lb_1,\lb_2\}$. Let \loss be a convex and Lipschitz continuous loss function. Then,
	\begin{align*}
		\normSup{\froba-\frobb} \le \frac{|\loss|_1}{\tau} \cdot  \bigg(& \kappa^2\cdot \norm{tv}{\P_1-\P_2} + \frac{\kappa^2}{\tau}\cdot |\lb_1-\lb_2|\\
		&+ \frac{1}{2} \cdot \normSup{\k_1-\k_2} + \kappa \cdot \sqrt{\normSup{\k_1-\k_2}} \bigg)\,.
	\end{align*}
\end{mythm}

\begin{mybem}
	The condition of $\H_1$ and $\H_2$ being separable can be difficult to check. Because of the separability of \X, this however holds true whenever $\k_1$ and $\k_2$ are continuous (\vgl \citeauthor{berlinet2004}, \citeyear{berlinet2004}, Corollary 4; \citeauthor{steinwart2008}, \citeyear{steinwart2008}, Lemma 4.33), and it suffices to verify this continuity instead (which is satisfied by most of the typically used kernels and is easy to check). 
\end{mybem}

Now, the subsequent \Cref{Thm:Rob_StabilityLp} states a result which is very similar to that from \Cref{Thm:Rob_Stability} but with respect to the \Lppix-norm: 

\begin{mythm}\label{Thm:Rob_StabilityLp}
	Let \X be a complete and separable metric space and $\Y\subseteq\R$ be closed. Let $\P_1,\P_2\in\MXY$ be probability measures, $\lb_1,\lb_2>0$ and $\k_1,\k_2$ be measurable and bounded kernels on \X with separable RKHSs $\H_1,\H_2$. Denote $\kappa:=\max\{\normSup{\k_1},\normSup{\k_2}\}$ and $\tau:=\min\{\lb_1,\lb_2\}$. Let \loss be a convex and Lipschitz continuous loss function. Then, for all $p\in[1,\infty)$ and all $i\in\{1,2\}$,
	\begin{align*}
		\normLppix{\froba-\frobb} \le \frac{|\loss|_1}{\tau} \cdot \bigg(& \kappa^2\cdot \norm{tv}{\P_1-\P_2} + \frac{\kappa^2}{\tau}\cdot |\lb_1-\lb_2|\\ 
		&+ \frac{1}{2} \cdot \normLppixb{\k_1-\k_2}+ \kappa \cdot \sqrt{\normLppixb{\k_1-\k_2}} \bigg)\,.
	\end{align*}
\end{mythm}

This result will become particularly useful in section \Cref{SubSec:Rob_Loc_DiffLoc} where we will investigate the stability of \it localized SVMs \rm by examining the difference between two localized SVMs that are based on two different regionalizations of the input space \X. These different regionalizations will lead to the two localized SVMs possibly vastly differing at some points where the regionalizations do not coincide and thus no interesting bound on the supremum norm of their difference being possible. On the other hand, if the regionalizations do not differ too much, we will still be able to derive meaningful bounds on the \Lapix-norm of the difference between the two localized SVMs based on \Cref{Thm:Rob_StabilityLp}.

%% file: Sections/03_Localization.tex
We now want to take look at localized SVMs and show that they inherit the stability properties from \Cref{Thm:Rob_Stability,Thm:Rob_StabilityLp} under certain conditions on the regionalization method. We will take a similar approach to \citet{dumpert2018,dumpert2020}, \ie divide the input space \X into several, possibly overlapping, regions with relatively mild assumptions about the specific regionalization. On these subspaces, we will define local SVMs which we will then combine in order to obtain a global predictor that we will call \it localized SVM\rm.

%% file: Sections/03a_Loc_GleicheAufteilung.tex
\subsection{Same regionalization for both localized SVMs}\label{SubSec:Rob_Loc_SameLoc}

First, we assume that both of the global predictors we want to compare in order to assess the stability of this localized approach are based on the same regionalization. Our stability result will not be based on any specific regionalization method but rather can be applied to any regionalization satisfying some very mild conditions. We will denote this regionalization by $\XBbold := \{\Xb[1],\dots,\Xb[B]\}$ for some sets $\Xb[1],\dots,\Xb[B]$ such that the following holds true:

\begin{description}
	\item[(R1)] $\Xb[1],\dots,\Xb[B] \subseteq \X$ and $\X=\bigcup_{b=1}^B \Xb$\,.
\end{description}
Additionally, \XBbold needs to satisfy the following condition for whichever probability measure \P we will base a localized SVM on:
\begin{description}
	\item[(R2)] $\P(\Xb\times\Y) > 0$ for all $b\in\{1,\dots,B\}$.
\end{description}
Note that \textbf{(R1)} tells us that the regions need not necessarily be pairwise disjoint but can instead also overlap. Condition \textbf{(R2)} of no region having probability mass zero is trivially needed because the local SVM on the respective region would not be defined elsewise.

Since we want to investigate stability similarly to \Cref{Thm:Rob_Stability,Thm:Rob_StabilityLp}, we again assume to have two possibly different Borel probability measures $\P_1$ and $\P_2$ on \XY which the predictors we want to compare are based on. For obtaining local SVMs on $\Xb[1],\dots,\Xb[B]$ for our localized approach, we now first need to introduce the associated local probability measures on these regions (respectively on $\Xb\times\Y$, $b=1,\dots,B$) by restricting $\P_1$ and $\P_2$: On an arbitrary $\tilde{\X}\subseteq\X$ satisfying $\P_i(\tilde{\X}\times\Y)>0$, we define the local probability measure based on $\P_i$ by
\begin{align}\label{eq:Rob_Loc_Pib}
	\P_{i,\tilde{\X}} := \frac{1}{\P_i(\tilde{\X}\times\Y)}\cdot (\P_i)\einschraenkung_{\tilde{\X}\times\Y}\,,\qquad i\in\{1,2\}.
\end{align}
If the regionalization satisfies \textbf{(R2)} for $\P_1$ and $\P_2$, these local probability measures are obviously well-defined on any $\Xb$, $b\in\{1,\dots,B\}$. Since we will mainly need the local probability measures on these regions, we also write $\Pib:=\P_{i,\Xb}$ to shorten the notation, $i\in\{1,2\}$, $b\in\{1,\dots,B\}$.

As mentioned in \Cref{Sec:Rob_Intro}, one big advantage of such localized approaches lies in their increased flexibility with regards to learning a function whose complexity and volatility vary across the input space since areas with differing complexities of the function can be separated into different regions. Hence, it should obviously be possible to choose different regularization parameters and kernels in the different regions because they to some extent control the complexity of the resulting SVM. We therefore have vectors of regularization parameters $\lbibold := (\lb_{i,1},\dots,\lb_{i,B})$, with $\lb_{i,b}>0$ for all $b\in\{1,\dots,B\}$, and vectors of kernels $\kibold := (\k_{i,1},\dots,\k_{i,B})$, for $i=1,2$. Based on these and a shifted loss function \lossshift, we obtain from \eqref{eq:Rob_DefShiftSVM} SVMs 
\begin{align*}
	\fibext : \Xb\to\R\,, \qquad i\in\{1,2\},\, b\in\{1,\dots,B\}
\end{align*}
which we call \it local SVMs \rm on \Xb.

For combining these local SVMs and thus obtaining a global predictor on \X, we first need to extend them in a way such that they are defined on all of \X. That is, for a function $g:\tilde{\X}\to\R$ on $\tilde{\X}\subseteq\X$, we define the zero-extension $\hat{g}:\X\to\R$ by
\begin{align*}
	\hat{g}(x) := \begin{cases}
		g(x) &\wenn x\in\tilde{\X}\,,\\
		0 &\sonst\,.
	\end{cases}
\end{align*}
Later on we will also need zero-extensions of kernels and probability measures which we indicate by using the $(\hat{\cdot})$-notation as well: If $\k:\tilde{\X}\times\tilde{\X}\to\R$ is a kernel on $\tilde{\X}\subseteq\X$, we define $\hat{\k}:\XX\to\R$ by
\begin{align*}
	\hat{\k}(x_1,x_2) := \begin{cases}
		k(x_1,x_2) &\wenn x_1,x_2\in\tilde{\X}\,,\\
		0 &\sonst\,.
	\end{cases}
\end{align*}
If \Q is a Borel probability measure on $\tilde{\X}\times\Y$ with $\tilde{\X}\subseteq\X$, we define $\hat{\Q}:\BXY\to[0,1]$ by
\begin{align*}
	\hat{\Q}(A) := \Q\left(A\cap\left(\tilde{\X}\times\Y\right)\right)\qquad \forall\, A\in\BXY\,.
\end{align*}

Before finally defining our global predictors, we lastly also need weight functions $w_b$, $b\in\{1,\dots,B\}$, which pointwisely control the influence of the different local SVMs in areas where two or more regions overlap. We require these weight functions to satisfy the same three conditions as in \citet{dumpert2018,dumpert2020}:
\begin{description}
	\item[(W1)] $w_b:\X\to[0,1]$ for all $b\in\{1,\dots,B\}$.
	\item[(W2)] $\sum_{b=1}^B w_b(x) = 1$ for all $x\in\X$.
	\item[(W3)] $w_b(x)=0$ for all $x\notin\Xb$ and all $b\in\{1,\dots,B\}$.
\end{description}

With this, global predictors \fiext[1] and \fiext[2], which we will call \it localized SVMs \rm even though they are not necessarily SVMs themselves, can be defined by
\begin{align}\label{eq:Rob_Loc_DefGlobalPredictor}
	\fiext: \X\to\R\,,\, x\mapsto \sum_{b=1}^{B} w_b(x) \cdot \fibextdach(x)
\end{align}
for $i=1,2$, where we again omitted the shifted loss function \lossshift from the index to shorten the notation.

The succeeding theorem states that \Cref{Thm:Rob_Stability} can be transferred to the situation at hand, \ie that such localized SVMs inherit a similar stability property from regular SVMs:

\begin{mythm}\label{Thm:Rob_Loc_GleicheAufteilung}
	Let \X be a complete and separable metric space and $\Y\subseteq\R$ be closed. Let $\P_1,\P_2\in\MXY$ be probability measures. Let $\XBbold:=\{\Xb[1],\dots,\Xb[B]\}$ be a regionalization of \X such that \XBbold satisfies \textbf{(R1)} and, for $\P_1$ as well as for $\P_2$, \textbf{(R2)}. For all $i\in\{1,2\}$ and $b\in\{1,\dots,B\}$, let $\lbib>0$ and let $\kib$ be a bounded and measurable kernel on \Xb with separable RKHS $\Hib$. Denote $\kappa_b:=\max\{\normSup{\kib[1]},\normSup{\kib[2]}\}$ and $\tau_b:=\min\{\lbib[1],\lbib[2]\}$ for all $b\in\{1,\dots,B\}$. Let \loss be a convex and Lipschitz continuous loss function. Let \fiext[1] and \fiext[2] be defined as in \eqref{eq:Rob_Loc_DefGlobalPredictor} with the weight functions $w_1,\dots,w_B$ satisfying \textbf{(W1)}, \textbf{(W2)} and \textbf{(W3)}. Then,
	\begin{align*}
		\normSup{\fiext[1]-\fiext[2]} \le |\loss|_1 \cdot \max_{b\in\{1,\dots,B\}} \frac{1}{\tau_b} \cdot  \bigg(& \kappa_b^2\cdot \norm{tv}{\Pib[1]-\Pib[2]} + \frac{\kappa_b^2}{\tau_b}\cdot |\lbib[1]-\lbib[2]|\\ 
		&+ \frac{1}{2} \cdot \normSup{\kib[1]-\kib[2]} + \kappa_b \cdot \sqrt{\normSup{\kib[1]-\kib[2]}} \bigg)\,.
	\end{align*}
\end{mythm}

Similarly, we can also transfer \Cref{Thm:Rob_StabilityLp} in order to bound the \Lppix-norm of the difference:

\begin{mythm}
	\label{Thm:Rob_Loc_GleicheAufteilungLp}
	Let \X be a complete and separable metric space and $\Y\subseteq\R$ be closed. Let $\P_1,\P_2\in\MXY$ be probability measures. Let $\XBbold:=\{\Xb[1],\dots,\Xb[B]\}$ be a regionalization of \X such that \XBbold satisfies \textbf{(R1)} and, for $\P_1$ as well as for $\P_2$, \textbf{(R2)}. For all $i\in\{1,2\}$ and $b\in\{1,\dots,B\}$, let $\lbib>0$ and let $\kib$ be a bounded and measurable kernel on \Xb with separable RKHS $\Hib$. Denote $\kappa_b:=\max\{\normSup{\kib[1]},\normSup{\kib[2]}\}$ and $\tau_b:=\min\{\lbib[1],\lbib[2]\}$ for all $b\in\{1,\dots,B\}$. Let \loss be a convex and Lipschitz continuous loss function. Let \fiext[1] and \fiext[2] be defined as in \eqref{eq:Rob_Loc_DefGlobalPredictor} with the weight functions $w_1,\dots,w_B$ satisfying \textbf{(W1)}, \textbf{(W2)} and \textbf{(W3)}. Then, for all $p\in[1,\infty)$ and all $i\in\{1,2\}$,
	\begin{align*}
		&\normLppix{\fiext[1]-\fiext[2]}\\ 
		&\le |\loss|_1 \cdot \sum_{b=1}^B \left(\Pix(\Xb)\right)^{1/p} \cdot  \bigg( \frac{\kappa_b^2}{\tau_b}\cdot \norm{tv}{\Pib[1]-\Pib[2]} + \frac{\kappa_b^2}{\tau_b^2}\cdot |\lbib[1]-\lbib[2]|\\
		&\hspace*{4.3cm}+ \frac{1}{2\tau_b} \cdot \normLppibxb{\kib[1]-\kib[2]} + \frac{\kappa_b}{\tau_b} \cdot \sqrt{\normLppibxb{\kib[1]-\kib[2]}} \bigg)\,.
	\end{align*}
\end{mythm}

\begin{mybem}
	\Cref{Thm:Rob_Loc_GleicheAufteilung} does actually not need \textbf{(W3)} and in \Cref{Thm:Rob_Loc_GleicheAufteilungLp} we can even waive \textbf{(W2)} as well as \textbf{(W3)}. We still included them in the assumptions of the two theorems since we think that weight functions should usually satisfy these conditions.
\end{mybem}

As can be seen from the two preceding theorems, localizing the SVMs does not ruin their stability with respect to probability measure, regularization parameters and kernel: We can still bound the difference between two localized SVMs---with respect to its supremum or an \Lp-norm---by a term that converges to zero whenever the norm of total variation between the two probability measures, the difference between the two regularization parameters and the supremum norm respectively the \Lp-norm of the difference between the two kernels all converge to zero on all regions of the regionalization that is used.

%% file: Sections/03b_Loc_VerschAuft.tex
\subsection{Different regionalizations for the localized SVMs}\label{SubSec:Rob_Loc_DiffLoc}

In the previous section, we investigated stability of localized SVMs with respect to changes in the triple $(\P,\lbbold,\kbold)$ but not in the regionalization. However, when there are changes in the distribution \P used for calculating the localized SVM (for example, because of changes in the training data set in practice), it may very well happen that this also affects the regionalization if it is not predetermined but also based on a learning method \citep[for example, decision trees, \vgl][among others]{bennett1998,wu1999,tibshirani2007,chang2010}. Hence, we will take a closer look at the effect of slight changes in the regionalization (in addition to those in $(\P,\lbbold,\kbold)$) on the resulting localized SVM in this section. 

First of all, it has to be mentioned that we will sadly not be able to derive a meaningful result regarding the supremum norm of the difference of two such localized SVMs (like \Cref{Thm:Rob_Loc_GleicheAufteilung} in the case of coinciding regionalizations), which can readily be seen from the simple example visualized in \Cref{Abb:Rob_Loc_DiffRegion}. There, two localized SVMs are being compared. Both of them are based on the same training data (that is, on the same empirical distribution) generated according to
\begin{align*}
	X \sim \unif[-1,1]\,,\qquad Y|X \sim \text{sign}(X) + \eps \text{ with } \eps \sim \normal[0,0.5]\,,
\end{align*} 
with $\unif[a,b]$ denoting the uniform distribution on $(a,b)$ and \normal[\mu,\sigma^2] the normal distribution with mean $\mu$ and variance $\sigma^2$. Furthermore, both localized SVMs use the same regularization parameter and the same kernel on every region. They only differ in the underlying regionalization: The input space is split into two parts in both cases, but for \fa the border between the two regions is at $x=0$ (thus exactly capturing the pattern in the data) whereas it is moved slightly to the right, to $x=0.05$, for \fb. 

\begin{figure}
	\begin{center}
		\vspace*{-0.5cm}\includegraphics[width=0.7\textwidth]{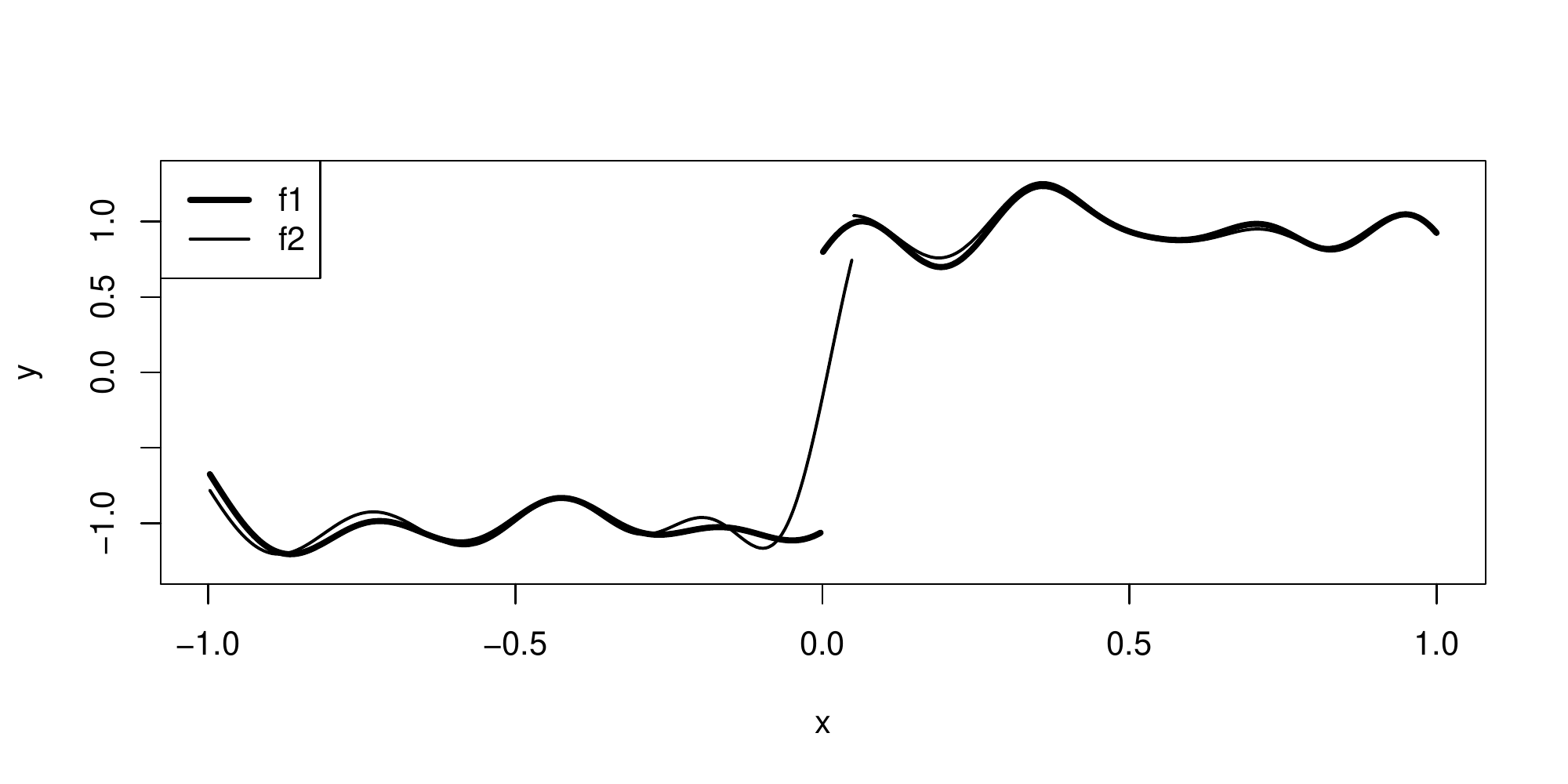}
		\vspace*{-0.5cm}\caption{Comparison of two localized SVMs based on the same distribution, regularization parameters and kernels, but on slightly different regionalizations.}
		\label{Abb:Rob_Loc_DiffRegion}
	\end{center}
\end{figure}

It can easily be seen from \Cref{Abb:Rob_Loc_DiffRegion} that this very minor change in the regionalization greatly impacts the maximum difference between \fa and \fb and it is thus obviously not possible to bound this maximum difference between two localized SVMs in any meaningful way. However, the same \Cref{Abb:Rob_Loc_DiffRegion} also suggests that it might still be possible to find such meaningful bounds on the \Lapix-norm of the difference (which is rather small in the example, approximately $0.06$, compared to the supremum norm of about $0.95$), similarly to \Cref{Thm:Rob_Loc_GleicheAufteilungLp}. This will indeed be the case, but before stating the corresponding theorem, we first need to modify some of the notation introduced in \Cref{SubSec:Rob_Loc_SameLoc} such that it fits this new situation:

First of all, we have two different regionalizations $\XiAibold[1] := \{\Xiaflex{1}{1},\dots,\Xiaflex{1}{A_1}\}$ and $\XiAibold[2] := \{\Xiaflex{2}{1},\dots,\Xiaflex{2}{A_2}\}$ now. Contrary to \Cref{SubSec:Rob_Loc_SameLoc}, these regionalizations are now required to actually be partitions of \X, \ie to satisfy, for $i=1,2$, the following modified version of \textbf{(R1)}:
\begin{description}
	\item[(R1')] $\Xiaflex{i}{1},\dots,\Xiaflex{i}{A_i}\subseteq\X$ and $\X=\bigcupdot_{a=1}^{A_i} \Xia$.
\end{description}
Furthermore, we also need to alter \textbf{(R2)} slightly since our proofs in this section will use auxiliary SVMs defined on all possible intersections of sets from \XiAibold[1] with sets from \XiAibold[2], which is why all these intersections need to have positive probability with respect to any probability measure \P which a localized SVM will be based on:
\begin{description}[labelwidth=\widthof{\bfseries (R2')},leftmargin=!]
	\item[(R2')] $\P(\Xbstar\times\Y)>0$ for all $\Xbstar \in \XBstarbold$, where
	\begin{align*}
		\XBstarbold := \left\{\Xbstar[1],\dots,\Xbstar[B]\right\} := \left\{\X^*\subseteq\X \,\left|\, \exists\, \Xiai[1]\in\XiAibold[1],\Xiai[2]\in\XiAibold[2] : \X^* = \Xiai[1] \cap \Xiai[2] \right.\right\} \setminus \{\emptyset\} \,.	
	\end{align*}
\end{description}

Based on \XBstarbold from \textbf{(R2')}, we denote $J_{i,a}:=\{b\in\{1,\dots,B\}\,|\,\Xbstar\subseteq\Xia\}\ne\emptyset$ for $i=1,2$ and $a=1,\dots,A_i$. Additionally, for $i=1,2$ and $b=1,\dots,B$, we denote by \aib that index $a\in\{1,\dots,A_i\}$ such that $\Xbstar\subseteq\Xia$, which is well-defined because of \textbf{(R2')} (existence of such an index) and \textbf{(R1')} (uniqueness of that index).

Local probability measures can be defined as before, \vgl \eqref{eq:Rob_Loc_Pib}, and we will again shorten the notation for the ones we will mainly use: $\Pia:=\P_{i,\Xia}$, $i=1,2$, $a=1,\dots,A_i$. We then obtain local SVMs
\begin{align*}
	\fiaext : \Xia\to\R\,,\qquad i\in\{1,2\},\, a\in\{1,\dots,A_i\}
\end{align*}
based on regularization parameters $\lbia>0$ and kernels $\kia$. 

For combining these local SVMs in order to obtain the global predictors, no weight functions are needed this time, since the regions from the regionalizations are not allowed to overlap, \vgl \textbf{(R1')}. Thus, the localized SVMs are now readily defined as 
\begin{align}\label{eq:Rob_Loc_DefGlobalPredictorDiffLoc}
	\fiextreg: \X\to\R\,,\, x\mapsto \sum_{a=1}^{A_i} \fiaextdach(x)
\end{align}
for $i=1,2$, where $\lbibold:=(\lb_{i,1},\dots,\lb_{i,A_i})$ and $\kibold:=(\k_{i,1},\dots,\k_{i,A_i})$, again omitting the shifted loss function \lossshift from the index to shorten the notation.

As before, we will investigate stability by bounding the norm of the difference of two estimators, in this case of two localized SVMs, based on how much the underlying probability measures and the vectors of regularization parameters and kernels differ. Additionally, the regionalizations are now added as a fourth possible difference. In order to base our new bound on the difference between the two underlying regionalizations as well, this difference first has to be quantified somehow. We will base this quantification on the intersections between regions from the different regionalizations, \ie on the sets $\Xbstar[1],\dots,\Xbstar[B]$ from \XBstarbold introduced in \textbf{(R2')}. For each such intersection \Xbstar, we look at two properties which both in some sense characterize the difference between the two regionalizations: Firstly, it has to be considered how much the two intersecting regions producing \Xbstar differ in size---relatively to these regions' own size and with respect to some probability measure \Q on \X satisfying $\Q(\Xia)>0$ for all $i\in\{1,2\}$ and $a\in\{1,\dots,A_i\}$. That is, we have to include the term
\begin{align}\label{eq:Rob_Loc_Versch_DiffRegPartI}
	\frac{\left|\Q(\Xiaib[1])-\Q(\Xiaib[2])\right|}{\max\left\{\Q(\Xiaib[1]),\Q(\Xiaib[2])\right\}}
\end{align}
for each such intersection, \ie for each $b\in\{1,\dots,B\}$. This difference in size has to be accounted for since a large difference could possibly lead to the local SVM on the smaller of the two regions being fitted much closer to its underlying data than its counterpart and the two local SVMs therefore greatly differing on the regions' intersection \Xbstar. Secondly, we also have to consider how closely each region from \XiAibold[1] as well as from \XiAibold[2] coincides with one of the intersections $\Xbstar[1],\dots,\Xbstar[B]$, \ie how well each such region can be identified with a region from the other regionalization, again with respect to some probability measure \Q on \X. In order to control this property, we will include the term
\begin{align}\label{eq:Rob_Loc_Versch_DiffRegPartII}
	\Q_{\Xiaib}(\Xbstar)\cdot \left(1-\Q_{\Xiaib}(\Xbstar)\right)
\end{align}
for each $b\in\{1,\dots,B\}$ and $i\in\{1,2\}$. If a region \Xia closely coincides with $\Xbstar[b_0]$ for some $b_0\in J_{i,a}$, then $\Q_{\Xia}(\Xbstar)$ will be either close to 1 or close to 0, and \eqref{eq:Rob_Loc_Versch_DiffRegPartII} will hence be small, for each $b\in J_{i,a}$. If this is the case for all $i\in\{1,2\}$ and $a\in\{1,\dots,A_i\}$, each region can be identified well with a region from the other regionalization and the two regionalizations are therefore similar to each other in the sense of this second criterion.

We now combine these two criteria in order to obtain a quantity to measure the difference between \XiAibold[1] and \XiAibold[2]. Since we analyzed the two criteria intersection-wise, we also define the quantification of this difference intersection-wise, \ie
\begin{align}\label{eq:Rob_Loc_Versch_DiffRegGesamt}
	\distReg{\Q}(\XiAibold[1],\XiAibold[2]) :=&\ \frac{\left|\Q(\Xiaib[1])-\Q(\Xiaib[2])\right|}{\max\left\{\Q(\Xiaib[1]),\Q(\Xiaib[2])\right\}}\notag\\ 
	&\ + \sum_{i=1}^{2} \Bigg(\frac{1}{2}\cdot\Q_{\Xiaib}(\Xbstar)\cdot \left(1-\Q_{\Xiaib}(\Xbstar)\right)\notag\\ 
	&\ \hspace*{2cm} + \sqrt{\Q_{\Xiaib}(\Xbstar)\cdot \left(1-\Q_{\Xiaib}(\Xbstar)\right)} \Bigg)\,,
\end{align}
for all $b\in\{1,\dots,B\}$ and for $\Q$ being a probability measure on \X satisfying $\Q(\Xia)>0$ for all $i\in\{1,2\}$ and $a\in\{1,\dots,A_i\}$. We additionally included the square root of the criterion from \eqref{eq:Rob_Loc_Versch_DiffRegPartII} since it will arise in the proof of the subsequent \Cref{Thm:Rob_Loc_VerschAuftLp} that this square root is also relevant to the difference we want to investigate.

Finally, we need to introduce some new notation which will arise in the succeeding theorem because of the already mentioned auxiliary SVMs on the sets $\Xbstar[1],\dots,\Xbstar[B]$ from \textbf{(R2')} that are needed for proving the theorem: We will denote the auxiliary distributions and kernels on the sets $\Xbstar[1],\dots,\Xbstar[B]$ by using the $(\cdot)^*$-notation. That is, $\Pibstar:=\P_{i,\Xbstar}$ and  $\kibstar:=\kiaib\einschraenkung_{\Xbstar\times\Xbstar}$ for $i=1,2$ and $b=1,\dots,B$. By \citet[Theorem 6]{berlinet2004}, \kibstar is actually a kernel (on \Xbstar) again.

With this, we can now state our stability result. As seen before, we will not be able to derive meaningful results regarding the supremum norm of the difference of two localized SVMs based on different regionalizations. However, it is indeed possible to obtain a meaningful bound on the \Lapix-norm of this difference:

\begin{mythm}\label{Thm:Rob_Loc_VerschAuftLp}
	Let \X be a complete and separable metric space and $\Y\subseteq\R$ be closed. Let $\P_1,\P_2\in\MXY$ be probability measures. Let $\XiAibold:=\{\Xiaflex{i}{1},\dots,\Xiaflex{i}{A_i}\}$, $i=1,2$, be regionalizations of \X such that \XiAibold[1] and \XiAibold[2] both satisfy \textbf{(R1')} and that they together satisfy, for $\P_1$ as well as for $\P_2$, \textbf{(R2')}. For all $i\in\{1,2\}$ and $a\in\{1,\dots,A_i\}$, let $\lbia>0$ and let $\kia$ be a bounded and measurable kernel on \Xia with separable RKHS $\Hia$. Denote $\kappa_b:=\max\{\normSup{\kiaib[1]},\normSup{\kiaib[2]}\}$, $\tau_b:=\min\{\lbiaib[1],\lbiaib[2]\}$ and $\rho_{1,b}:=\max\left\{\Pix[1](\Xiaib[1]),\Pix[1](\Xiaib[2])\right\}$ for all $b\in\{1,\dots,B\}$. Let \loss be a convex and Lipschitz continuous loss function. Let \fiextreg[1] and \fiextreg[2] be defined as in \eqref{eq:Rob_Loc_DefGlobalPredictorDiffLoc}. Then,
	\begin{align*}
		&\normLapix[1]{\fiextreg[1]-\fiextreg[2]}\\
		&\le |\loss|_1 \cdot \sum_{a=1}^{A_{2}} \Pix[1](\Xia[2]) \cdot \frac{\normSup{\kia[2]}^2}{\lbia[2]} \cdot \norm{tv}{\Pjai[2]{1}-\Pjai[2]{2}}\\
		&\hspace*{0.4cm} + |\loss|_1 \cdot \sum_{b=1}^B \bigg( \rho_{1,b}\cdot\frac{\kappa_b^2}{\tau_b^2} \cdot \left|\lbiaib[1] - \lbiaib[2] \right|\\
		&\hspace*{2.5cm} + \Pix[1](\Xbstar) \cdot \bigg(\frac{1}{2\tau_b} \cdot \normLapibstarxb[1]{\kibstar[1,b]-\kibstar[2,b]}\\ 
		&\hspace*{5cm}+ \frac{\kappa_b}{\tau_b} \cdot \sqrt{\normLapibstarxb[1]{\kibstar[1,b]-\kibstar[2,b]}} \bigg)\\
		&\hspace*{2.5cm} + \rho_{1,b} \cdot \frac{\kappa_b^2}{\tau_b} \cdot \distReg{\Px_1}(\XiAibold[1],\XiAibold[2]) \bigg)\,.
	\end{align*}
\end{mythm}

Note that the denominator occurring in $\distReg{\Px_1}(\XiAibold[1],\XiAibold[2])$ is greater than zero for all $b\in\{1,\dots,B\}$ in the situation of this theorem because of \XiAibold[1] and \XiAibold[2] being assumed to satisfy \textbf{(R2')} for $\P_1$, and the bound from the theorem is therefore well-defined.

By interchanging the roles of \fiextreg[1] and \fiextreg[2], it is furthermore obvious that \Cref{Thm:Rob_Loc_VerschAuftLp} also holds true with respect to the \Lapix[2]-norm if the indices on the right hand side are adjusted accordingly. For the sake of notational clarity, we did not explicitly include this in the theorem.

Even though allowing for differing regionalizations makes this result on total stability look more complicated than those from \Cref{SubSec:Rob_Loc_SameLoc} on first glance, the statement basically stays the same---with the main difference being that we can only bound the \La-norm in a meaningful way, but not other \Lp-norms or the supremum norm (for other \Lp-norms, it would be possible to derive a similar result, but in this case, the factor in front of the difference between the regularization parameters could increase with increasing similarity of the two regionalizations and it would therefore not necessarily be possible to interpret the result as yielding total stability, \ie particularly stability with respect to simultaneous slight changes in regularization parameters and regionalization). 
Here, we additionally need to consider the difference between the two regionalizations, but otherwise have the same statement as before: The \La-norm of the difference between the two localized SVMs converges to zero if, on all regions respectively intersections of regions, the norm of total variation of the difference between the two probability measures, the difference between the two regularization parameters, the \La-norm of the difference between the two kernels, and now additionally the difference between the two regionalizations, as measured by $\distReg{\Pix}$, all converge to zero as well.

%% file: Sections/04_Discussion.tex
This paper is composed of two main parts. In the first one, stability of SVMs with respect to slight changes in the full triple $(\P,\lb,\k)$, consisting of a probability measure \P, a regularization parameter \lb and a kernel \k, was investigated. This part is related to \citet{christmann2018}, where the difference \normSup{\froba-\frobb} between two such SVMs based on slightly differing triples $(\P_i,\lb_i,\k_i)$, $i=1,2$, had already been bounded in a very similar way. We succeeded in considerably generalizing the referenced result by \citet{christmann2018}, such that we now know that the investigated notion of stability holds true for any SVM that uses a Lipschitz continuous loss function. We also derived an analogous stability result regarding \normLppix{\froba-\frobb}, $i=1,2$, before turning our attention to the second part of this paper.

Here, we investigated localized SVMs which, amongst other advantages, thrive on their reduced computational requirements compared to calculating a global SVM. They share this advantage with other methods mentioned in the introduction, like for example distributed learning. Distributed learning is similar to localized approaches in that both divide the training data into several subsets and then produce a global predictor by combining the predictors obtained on the subsets. However, whereas the subsets constitute subregions of the input space in localized approaches, they are usually generated by drawing simple random samples (without replacement) from the original data set and thus typically cover almost the entire input space in distributed learning (that is, each of the predictors on the subsets is defined on all of \X and they are then typically combined by means of some weighted average). This leads to distributed learning reducing the computation time even further on the one hand (since the effort of regionalizing can be omitted) but not sharing the additional advantages of localized approaches (the ability to treat different structures in different regions of the input space in different ways and the ability to better model discontinuities in the true function) on the other hand, which is one reason why we think that localized learning can be interesting.

We managed to transfer our stability results to localized SVMs, adding to the list of properties such localized SVMs inherit from the local SVMs they are based on. This further substantiates the theoretical justification for localized SVMs to be used in order to accurately predict a function whose complexity varies across the input space or whenever a large data set drastically increases the computation time of a global SVM. It has even been possible to show stability with respect to the $\Lapix$-norm, $i=1,2$, if not only the triple $(\P,\lbbold,\kbold)$ but also the regionalization slightly changes. Since variations in the underlying probability measure (respectively data set)---which are also of interest in considerations regarding classical statistical robustness, where only stability with respect to the probability measure is regarded---may very well lead to changes in the regionalization, it is especially reassuring to see that this does not ruin the localized SVMs' stability (as long as a statistically robust method is used for constructing the regionalization, such that small changes in \P only lead to small changes in the regionalization).

Based on this influence of the probability measure on the regionalization, it might be interesting to take another look at the already existing results about localized SVMs' consistency, learning rates and classical statistical robustness and examine whether they still hold true if this influence is factored in, respectively what assumptions about the dependence between probability measure and regionalization are necessary in order for the results to still hold true.

%% file: Sections/A02_Aux_Stability.tex
In order to prove \Cref{Thm:Rob_Stability}, we will first apply the triangle inequality in order to decompose the difference which we have to bound:
\begin{align}\label{eq:Rob_Aux_Decomposition}
	\normSup{\froba-\frobb} \le&\ \normSup{\froba-f_{\P_2,\lb_1,\k_1}} + \normSup{f_{\P_2,\lb_1,\k_1}-f_{\P_2,\lb_2,\k_1}}\notag\\ 
	&\ + \normSup{f_{\P_2,\lb_2,\k_1}-\frobb}\,.
\end{align}
Of course, the order of this decomposition can also be varied. We will take this into account when actually proving \Cref{Thm:Rob_Stability} in \Cref{Sec:Rob_Proof_Stab}, but for now we will just investigate the three summands on the right hand side of \eqref{eq:Rob_Aux_Decomposition} separately. The \Lppix-norm of $\froba-\frobb$ can obviously be decomposed in the same way as 
\begin{align}\label{eq:Rob_Aux_DecompositionLp}
	\normLppix{\froba-\frobb} \le&\ \normLppix{\froba-f_{\P_2,\lb_1,\k_1}}\notag + \normLppix{f_{\P_2,\lb_1,\k_1}-f_{\P_2,\lb_2,\k_1}}\notag\\
	&\ + \normLppix{f_{\P_2,\lb_2,\k_1}-\frobb}
\end{align}
and thus, \Cref{Thm:Rob_StabilityLp} can also be proven by examining the three summands separately.

Before doing this, first for the supremum norm and then for the \Lppix-norm, we have to state two auxiliary results needed for conducting the proofs. Firstly, we recall a representer theorem for SVMs \citep[Theorem 7]{christmann2009} and prior to that the definition of a \it subdifferential \rm \citep[\vgl][]{phelps1993,christmann2009}, which is referenced in the representer theorem:

\begin{mydef}
	Let $E$ be a Banach space and let $f:E\to\R\cup\{\infty\}$ be a convex function, and $w\in E$ with $f(w)<\infty$. Then, the \it subdifferential \rm of $f$ at $w$ is defined by
	\begin{align*}
		\partial f(w) := \{w'\in E' : \langle w',v-w\rangle \le f(v)-f(w) \text{ for all } v\in E \}\,.
	\end{align*}
\end{mydef}

\begin{mybem}
	For a convex loss function $\loss:\XYR\to[0,\infty)$ we denote by $\partial \loss(x,y,t_0)$ the subdifferential with respect to the third argument, \ie the subdifferential of the convex function defined by $t\mapsto\loss(x,y,t)$ at the point $t_0\in\R$. We say that a function $g:\XY\to\R$ is from the subdifferential of \loss with respect to a function $f:\X\to\R$ if $g(x,y)\in\partial\loss(x,y,f(x))$ for all $(x,y)\in\XY$. We use an analogous notation for shifted loss functions \lossshift.
\end{mybem}

\begin{mythm}\label{Thm:Rob_Aux_Representer}
	\rm\citep{christmann2009}. \it Let \X be a complete and separable metric space and $\Y\subseteq\R$ be closed. Let $\P\in\MXY$ be a probability measure. Let \loss be a convex and Lipschitz continuous loss function, \k be a bounded and measurable kernel on \X with separable RKHS \H. Then, for all $\lb>0$, there exists an $h\in \L{\infty}(\P)$ such that
	\begin{align*}
		h(x,y) &\in \partial \lossshift(x,y,\fshifttheok(x)) \qquad \forall\, (x,y)\in\XY\\
		\fshifttheok &= -\frac{1}{2\lb}\ew[\P]{h\Phi}\\
		\normSup{h} &\le |\loss|_1\\
		\normH{\fshifttheok - f_{\lossshift,\bar{\P},\lb,\k}} &\le \frac{1}{\lb} \normH{\ew[\P]{h\Phi} - \ew[\bar{\P}]{h\Phi}}
	\end{align*}
	for all distributions $\bar{\P}$ on \XY.
\end{mythm}

Since the feature map $\Phi$ is \H-valued, we need to consider \H-valued \it Bochner integrals \rm when examining the expectations from \Cref{Thm:Rob_Aux_Representer} and similar integrals. For a detailed introduction to Bochner integrals, see \citet{diestel1977,diestel1984,denkowski2003}. We will additionally need the two succeeding inequalities \eqref{eq:Rob_Aux_BochnerInf} and \eqref{eq:Rob_Aux_BochnerLp} in order to bound the norm of a Bochner integral. Even though we suppose that these two inequalities are already established, we did not find them in the literature, which is why we prove them here:

\begin{mylem}\label{Lem:Rob_Aux_NormIntegralTauschen}
	Let \Q be a probability measure on some measurable space $(\Omega,\A)$ and let \k be a bounded kernel on $\Omega$ with RKHS \H. Let $g:\Omega\to\H$ be a \Q-Bochner integrable function. Then,
	\begin{align}\label{eq:Rob_Aux_BochnerInf}
		\normSup{\int_{\Omega} g(x)\,d\Q(x)} \le \int_{\Omega} \normSup{g(x)}\,d\Q(x)
	\end{align}
	and, for all $p\in[1,\infty)$,
	\begin{align}\label{eq:Rob_Aux_BochnerLp}
		\norm{\Lp(\Q)}{\int_{\Omega} g(x)\,d\Q(x)} \le \int_{\Omega} \norm{\Lp(\Q)}{g(x)}\,d\Q(x)\,.
	\end{align}
\end{mylem}

\begin{proof}
	By \citet[Definition 3.10.7]{denkowski2003}, $g$ being \Q-Bochner integrable means that there exists a sequence $(s_n)_{n\in\N}$ of so-called simple functions $s_n:\Omega\to\H,\, \omega\mapsto \sum_{j=1}^{m_n} b_j^{(n)}\Ind[A_j^{(n)}](\omega)$, with $b_j^{(n)}\in\H$, $A_j^{(n)}\in\A$ and $\Ind[A_j^{(n)}]$ denoting the indicator function on $A_j^{(n)}$ for all $n\in\N$ and $j\in\{1,\dots,m_n\}$, such that 
	\begin{align}\label{eq:Rob_Aux_NormalIntegralTauschen_a}
		\limn \int_{\Omega} \normH{g(\omega)-s_n(\omega)}\,d\Q(\omega) = 0\,.
	\end{align}
	Then, the same definition tells us that
	\begin{align*}
		\int_{\Omega} g(\omega)\, d\Q(\omega) := \limn \int_{\Omega} s_n(\omega) \, d\Q(\omega)\,,
	\end{align*}
	where 
	\begin{align*}
		\int_{\Omega} s_n(\omega) \, d\Q(\omega) := \sum_{j=1}^{m_n} b_j^{(n)}\Q\left(A_j^{(n)}\right)
	\end{align*}
	for all $n\in\N$. Additionally, we know from \citet[Chapter IV]{diestel1984} that we can without loss of generality assume $A_1^{(n)},\dots,A_{m_n}^{(n)}$ to be pairwise disjoint for all $n\in\N$. 

	Let now $\normArb{\cdot}$ denote either $\normSup{\cdot}$ or $\norm{\Lp(\Q)}{\cdot}$. Then,
	\begin{align}\label{eq:Rob_Aux_NormalIntegralTauschen_b}
		&\normArb{\int g(\omega)\,d\Q(\omega)} = \normArb{\limn\left(\sum_{j=1}^{m_n}b_j^{(n)}\Q\left(A_j^{(n)}\right)\right)} = \limn \normArb{\sum_{j=1}^{m_n}b_j^{(n)}\Q\left(A_j^{(n)}\right)}\notag\\
		&\le \limn \left(\sum_{j=1}^{m_n} \normArb{b_j^{(n)}} \Q\left(A_j^{(n)}\right)\right) = \limn \left(\sum_{j=1}^{m_n} \int \normArb{b_j^{(n)}} \Ind[A_j^{(n)}](\omega)\,d\Q(\omega)\right)\notag\\
		&= \limn \left(\int \sum_{j=1}^{m_n} \normArb{b_j^{(n)}} \Ind[A_j^{(n)}](\omega)\,d\Q(\omega)\right) = \limn \left(\int \normArb{\sum_{j=1}^{m_n} b_j^{(n)} \Ind[A_j^{(n)}](\omega)}\,d\Q(\omega)\right)\notag\\
		&= \limn \left(\int \normArb{s_n(\omega)}\,d\Q(\omega)\right) = \int \normArb{g(\omega)}\,d\Q(\omega)\,,
	\end{align}
	where we applied the continuity of $\normArb{\cdot}$ as a function on \H in the second step and the pairwise disjointness of $A_1^{(n)},\dots,A_{m_n}^{(n)}$ in the second to last row, with the continuity of $\normArb{\cdot}$ holding true because of $\norm{\Lp(\Q)}{h}\le\normSup{h}\le\normSup{k}\normH{h}$ and thus
	\begin{align}\label{eq:Rob_Aux_NormalIntegralTauschen_c}
		\normArb{h}\le\normSup{k}\normH{h}
	\end{align}
	for all $h\in\H$, \vgl \eqref{eq:Rob_InfNormHNorm}. Additionally, the equality in the last step of \eqref{eq:Rob_Aux_NormalIntegralTauschen_b} holds true because of
	\begin{align}\label{eq:Rob_Aux_NormalIntegralTauschen_d}
		&\left|\int \normArb{g(\omega)}\,d\Q(\omega) - \limn \left(\int \normArb{s_n(\omega)}\,d\Q(\omega)\right)\right| \le \limn \left(\int \big|\normArb{g(\omega)} - \normArb{s_n(\omega)}\big|\,d\Q(\omega)\right)\notag\\
		&\le \limn \left(\int \normArb{g(\omega)-s_n(\omega)}\,d\Q(\omega)\right) \le \normSup{k} \cdot \limn \left(\int \normH{g(\omega)-s_n(\omega)}\,d\Q(\omega)\right) = 0\,.
	\end{align}
	Here, we employed the finiteness of the two summands on the left hand side in the first step, and the reverse triangle inequality, \eqref{eq:Rob_Aux_NormalIntegralTauschen_c} and \eqref{eq:Rob_Aux_NormalIntegralTauschen_a} in the remaining steps. In the first step, the finiteness of the first summand follows directly from \eqref{eq:Rob_Aux_NormalIntegralTauschen_c} and Theorem 3.10.9 from \citet{denkowski2003}, and the finiteness of the second one can be shown by again using \eqref{eq:Rob_Aux_NormalIntegralTauschen_c} and then slightly adapting the mentioned theorem's proof:
	\begin{align*}
		&\limn \left(\int \normH{s_n(\omega)}\,d\Q(\omega)\right) \le \limn \left(\int \normH{s_n(\omega)-g(\omega)}\,d\Q(\omega) + \int\normH{g(\omega)}\,d\Q(\omega) \right)\\
		&= \limn \left(\int \normH{s_n(\omega)-g(\omega)}\,d\Q(\omega)\right) + \int\normH{g(\omega)}\,d\Q(\omega) = \int\normH{g(\omega)}\,d\Q(\omega) < \infty
	\end{align*}
	with the first inequality holding true because of the second integral on its right hand side being finite \citep[Theorem 3.10.9]{denkowski2003} and the first one being finite for $n$ sufficiently large, \vgl \eqref{eq:Rob_Aux_NormalIntegralTauschen_a}. The same equation \eqref{eq:Rob_Aux_NormalIntegralTauschen_a} also tells us that $\limn \left(\int \normH{s_n(\omega)-g(\omega)}\,d\Q(\omega)\right)$ exists and the linearity of the limit can therefore be applied in the second step. Finally, \eqref{eq:Rob_Aux_NormalIntegralTauschen_a} and the mentioned Theorem 3.10.9 yield the last two steps.
\end{proof}

We will now turn our attention to the three summands on the right hand side of \eqref{eq:Rob_Aux_Decomposition}. That is, in the subsequent three lemmas, we will examine the effect of only one element of the triple $(\P,\lb,\k)$ varying at a time, with the proofs of \Cref{Lem:Rob_Aux_P,Lem:Rob_Aux_lambda} being closely connected to their counterparts by \citet{christmann2018} but generalizing them to the case of non-differentiable losses.

\begin{mylem}\label{Lem:Rob_Aux_P}
	Let \X be a complete and separable metric space and $\Y\subseteq\R$ be closed. Let $\P_1,\P_2\in\MXY$ be probability measures, $\lb>0$ and $\k$ be a bounded and measurable kernel on \X with separable RKHS \H. Let \loss be a convex and Lipschitz continuous loss function. Then,
	\begin{align*}
		\normSup{f_{\P_1,\lb,\k}-f_{\P_2,\lb,\k}} \le  \frac{\normSup{\k}^2 |\loss|_1}{\lb}\cdot \norm{tv}{\P_1-\P_2}\,.
	\end{align*}
\end{mylem}

\begin{proof}
	First of all, 
	\begin{align*}
		\normSup{f_{\P_1,\lb,\k}-f_{\P_2,\lb,\k}} \le \normSup{\k} \cdot \normH{f_{\P_1,\lb,\k}-f_{\P_2,\lb,\k}}
	\end{align*}
	by \eqref{eq:Rob_InfNormHNorm}. By \Cref{Thm:Rob_Aux_Representer}, there exists a function $h$ from the subdifferential of \lossshift with respect to $f_{\P_1,\lb,\k}$ such that \citep[using the properties of Bochner integrals, \vgl][Lemma 6.1, as well as \eqref{eq:Rob_FeatureMap}]{christmann2018}
	\begin{align*}
		&\normH{f_{\P_1,\lb,\k}-f_{\P_2,\lb,\k}}\\
		&\le \frac{1}{\lb} \cdot \normH{\int h(x,y)\Phi(x) \,d\P_1(x,y) - \int h(x,y)\Phi(x) \,d\P_2(x,y)}\\
		&\le \frac{1}{\lb} \cdot \int \normH{h(x,y)\Phi(x)}\, d|\P_1-\P_2|(x,y)\\
		&\le \frac{1}{\lb} \cdot \sup_{(x,y)\in\XY} |h(x,y)| \cdot \sup_{x\in\X} \normH{\Phi(x)} \cdot \int  1 \,d|\P_1-\P_2|(x,y)\\
		&= \frac{1}{\lb} \cdot \sup_{(x,y)\in\XY} |h(x,y)| \cdot \sup_{x\in\X} \sqrt{\k(x,x)} \cdot \norm{tv}{\P_1-\P_2}\\
		&\le \frac{1}{\lb} \cdot |\loss|_1 \cdot \normSup{\k} \cdot \norm{tv}{\P_1-\P_2}\,,
	\end{align*}
	from which the assertion follows.
\end{proof}

\begin{mylem}\label{Lem:Rob_Aux_lambda}
	Let \X be a complete and separable metric space and $\Y\subseteq\R$ be closed. Let $\P\in\MXY$ be a probability measure, $\lb_1,\lb_2>0$ and $\k$ be a bounded and measurable kernel on \X with separable RKHS \H. Let \loss be a convex and Lipschitz continuous loss function. Then,
	\begin{align*}
		\normSup{f_{\P,\lb_1,\k}-f_{\P,\lb_2,\k}} \le  \frac{\normSup{\k}^2 |\loss|_1}{\min\{\lb_1,\lb_2\}^2}\cdot |\lb_1-\lb_2|\,.
	\end{align*}
\end{mylem}

\begin{proof}
	To shorten the notation, we define $\fj:=f_{\P,\lb_i,\k}$, $i=1,2$, in this proof. By \eqref{eq:Rob_InfNormHNorm} we know that
	\begin{align*}
		\normSup{\fa-\fb} \le \normSup{\k} \cdot \normH{\fa-\fb}\,.
	\end{align*}
	Assume now without loss of generality that $\normH{\fa-\fb}>0$ since the case $\normH{\fa-\fb}=0$ is trivial.
	 
	\Cref{Thm:Rob_Aux_Representer} yields functions $h_1$ and $h_2$ from the subdifferential of \lossshift (with respect to \fa respectively \fb) such that
	\begin{align*}
		\fa-\fb = -\frac{1}{2\lb_1} \cdot \int h_1(x,y)\Phi(x)\,d\P(x,y) + \frac{1}{2\lb_2} \cdot \int h_2(x,y)\Phi(x)\,d\P(x,y)\,.
	\end{align*}
	From this we obtain, by applying the reproducing property \eqref{eq:Rob_ReprodProperty} in the last step,
	\begin{align}\label{eq:Rob_Aux_lambda_ReprodAufteilung}
		\normH{\fa-\fb}^2 =&\ \langle \fa-\fb,\fa-\fb\rangle_\H\notag\\
		=&\ \left\langle \frac{1}{2\lb_2} \cdot \int h_2(x,y)\Phi(x)\,d\P(x,y), \fa-\fb\right\rangle_\H\notag\\
		&\ - \left\langle \frac{1}{2\lb_1} \cdot \int h_1(x,y)\Phi(x)\,d\P(x,y), \fa-\fb\right\rangle_\H\notag\\
		=&\ \frac{1}{2\lb_2} \cdot \int h_2(x,y)(\fa(x)-\fb(x))\,d\P(x,y) - \frac{1}{2\lb_1} \cdot \int h_1(x,y)(\fa(x)-\fb(x))\,d\P(x,y)\,.
	\end{align}
	Because \loss (and thus also \lossshift) is convex and $h_i(x,y)\in\partial\lossshift(x,y,\fj(x))$ for all $(x,y)\in\XY$ and for $i=1,2$, we know that
	\begin{align*}
		h_i(x,y)\cdot (t-\fj(x)) \le \lossshift(x,y,t) - \lossshift(x,y,\fj(x)) \qquad \forall\, t\in\R,\qquad i=1,2\,,
	\end{align*}
	more specifically
	\begin{align*}
		h_1(x,y)\cdot (\fb(x)-\fa(x)) \le \lossshift(x,y,\fb(x)) - \lossshift(x,y,\fa(x))
	\end{align*}
	and
	\begin{align*}
		h_2(x,y)\cdot (\fa(x)-\fb(x)) \le \lossshift(x,y,\fa(x)) - \lossshift(x,y,\fb(x))\,.
	\end{align*}
	Plugging these two inequalities into \eqref{eq:Rob_Aux_lambda_ReprodAufteilung} yields
	\begin{align}\label{eq:Rob_Aux_lambda_Absch}
		\normH{\fa-\fb}^2 &\le \left(\frac{1}{2\lb_2}-\frac{1}{2\lb_1}\right) \cdot \int \lossshift(x,y,\fa(x)) - \lossshift(x,y,\fb(x)) \,d\P(x,y)\notag\\
		&= \left(\frac{1}{2\lb_2}-\frac{1}{2\lb_1}\right) \cdot \big(\ew[\P]{\lossshift(X,Y,\fa(X))} -  \ew[\P]{\lossshift(X,Y,\fb(X))} \big)
	\end{align}
	Now, $\normH{\fa-\fb}^2$ being positive implies that the right hand side of this inequality has to be positive as well, \ie both factors need to have the same sign. First assume $\lb_1>\lb_2$:
	
	In this case $\frac{1}{2\lb_2}-\frac{1}{2\lb_1}> 0$ and thus $\ew[\P]{\lossshift(X,Y,\fa(X))} -  \ew[\P]{\lossshift(X,Y,\fb(X))}$ has to be positive as well. Because of the definition of \fa as the minimizer of the regularized risk with regularization parameter $\lb_1$, we know that
	\begin{align*}
		\ew[\P]{\lossshift(X,Y,\fa(X))} + \lb_1 \normH{\fa}^2 \le \ew[\P]{\lossshift(X,Y,\fb(X))} + \lb_1 \normH{\fb}^2\,.
	\end{align*}
	From this, it follows that
	\begin{align*}
		0 &< \ew[\P]{\lossshift(X,Y,\fa(X))} -  \ew[\P]{\lossshift(X,Y,\fb(X))} \le \lb_1 \cdot \left(\normH{\fb}^2-\normH{\fa}^2 \right)\\
		&= \lb_1 \cdot \left(\normH{\fa} + \normH{\fb} \right) \cdot \left(\normH{\fb} - \normH{\fa} \right) \le \lb_1 \cdot \left(\normH{\fa} + \normH{\fb} \right) \cdot \normH{\fa-\fb}
	\end{align*}
	with the last inequality holding true because of $\lb_1  (\normH{\fa} + \normH{\fb} )\ge 0$ and the reverse triangle inequality. Plugging this into \eqref{eq:Rob_Aux_lambda_Absch} and dividing by $\normH{\fa-\fb}$, we obtain
	\begin{align}\label{eq:Rob_Aux_lambda_AbschB}
		\normH{\fa-\fb} \le  \frac{1}{2}\cdot \left(\frac{\max\{\lb_1,\lb_2\}}{\min\{\lb_1,\lb_2\}} -1\right)\cdot \left(\normH{\fa} + \normH{\fb} \right)\,.
	\end{align}
	The case $\lb_2>\lb_1$ yields the same inequality. 
	
	By additionally applying that $\normH{\fj}\le \frac{1}{\lb_i}|\loss|_1\normSup{\k}$ \citep[\vgl][proof of Proposition 3]{christmann2009}, $i=1,2$, we now obtain
	\begin{align*}
		\normH{\fa-\fb} &\le \frac{|\loss|_1\normSup{\k}}{2} \cdot \left(\frac{\max\{\lb_1,\lb_2\}}{\min\{\lb_1,\lb_2\}} -1\right) \cdot \left(\frac{1}{\lb_1} + \frac{1}{\lb_2} \right)\\
		&\le \frac{|\loss|_1\normSup{\k}}{2\min\{\lb_1,\lb_2\}} \cdot \big(\max\{\lb_1,\lb_2\} - \min\{\lb_1,\lb_2\}\big) \cdot \frac{2}{\min\{\lb_1,\lb_2\}}\\
		&= \frac{|\loss|_1\normSup{\k}}{\min\{\lb_1,\lb_2\}^2} \cdot |\lb_1-\lb_2|
	\end{align*}
	which yields the assertion.
\end{proof}

\begin{mylem}\label{Lem:Rob_Aux_Kern}
	Let \X be a complete and separable metric space and $\Y\subseteq\R$ be closed. Let $\P\in\MXY$ be a probability measure, $\lb>0$ and $\k_1,\k_2$ be bounded and measurable kernels on \X with separable RKHSs $\H_1,\H_2$. Denote $\kappa:=\max\{\normSup{\k_1},\normSup{\k_2}\}$. Let \loss be a convex and Lipschitz continuous loss function. Then,
	\begin{align*}
		\normSup{f_{\P,\lb,\k_1}-f_{\P,\lb,\k_2}} \le  \frac{|\loss|_1}{\lb} \cdot \left(\frac{1}{2} \cdot \normSup{\k_1-\k_2} + \kappa \cdot \sqrt{\normSup{\k_1-\k_2}}\right)\,.
	\end{align*}
\end{mylem}

In order to prove \Cref{Lem:Rob_Aux_Kern}, we first need a short auxiliary statement which is probably well-known, but which we were unable to find a reference for. For reasons of better readability we therefore prove the following auxiliary lemma:

\begin{mylem}\label{Lem:Rob_Aux_RKHS}
	Let $\X\ne\emptyset$ and let $\k:\X\times\X\to\R$ be a kernel with RKHS \H. Let $\alpha>0$ and define the kernel $\ktilde:\X\times\X\to\R$ by $\ktilde := \alpha\k$. Then, $\Htilde:=\H$ equipped with the norm $\norm{\Htilde}{\cdot} := \frac{1}{\sqrt{\alpha}}\normH{\cdot}$ is the RKHS of \ktilde.
\end{mylem}

\begin{proof}
	\citet[Lemma 4.5]{steinwart2008} yields that \ktilde is actually a kernel. Hence, the assertion follows directly from \citet[Theorem 4.21]{steinwart2008}: It can easily be seen for the pre-Hilbert space from equation (4.12) in that theorem and follows by completion for the whole Hilbert space.
\end{proof}

\begin{proof}[Proof of \Cref{Lem:Rob_Aux_Kern}]
	To shorten the notation, we define $\fj:=f_{\P,\lb,\k_i}$, $i=1,2$, in this proof. 
	
	Define $\ktilde_i:=\frac{\k_i}{2}$ for $i=1,2$. From \Cref{Lem:Rob_Aux_RKHS} we know that $\Htilde_i=\H_i$ (equipped with the norm $\norm{\Htilde_i}{\cdot}=\sqrt{2}\,\norm{\H_i}{\cdot}$) is the RKHS of $\ktilde_i$. Thus, we obviously have $\fj\in\Htilde_i$ for $i=1,2$.
	
	In the next step, we define a new space which contains \fa as well as \fb by 
	\begin{align*}
		\tilde{\H} := \tilde{\H}_1 \oplus \tilde{\H}_2 := \left\{g:\X\to\R \,\left|\, g=g_1+g_2, g_1\in\tilde{\H}_1, g_2\in\tilde{\H}_2 \right.\right\}\,.
	\end{align*}
	\citet[Theorem 5]{berlinet2004} tells us that $\tilde{\H}$ equipped with the norm
	\begin{align*}
		\norm{\Htilde}{g}^2:=\min_{g_1\in\Htilde_1,\,g_2\in\Htilde_2\,:\, g_1+g_2=g} \left(\norm{\Htilde_1}{g_1}^2+\norm{\Htilde_2}{g_2}^2\right) \qquad \forall\, g\in\Htilde 
	\end{align*}
	is the RKHS of the reproducing kernel $\tilde{\k}:=\tilde{\k}_1+\tilde{\k}_2=(\k_1+\k_2)/2$. Since obviously $\fa,\fb\in\tilde{\H}$, we will now use this new RKHS as an aid for investigating the difference between \fa and \fb:
	
	First of all, because $\ktilde$ is measurable and bounded by $||\ktilde||_\infty\le\frac{1}{2}\left(\normSup{\k_1}+\normSup{\k_2}\right)<\infty$ and \Htilde is obviously separable, there exists a unique SVM $f_{\lossshift,\P,\lb,\tilde{\k}}=:\tilde{f}$ (\vgl \Cref{Thm:Rob_Aux_Representer}). The triangle inequality then yields 
	\begin{align}\label{eq:Rob_Aux_Kern_Triangle}
		\normSup{\fa-\fb} \le \normSup{\fa-\ftilde} + \normSup{\fb-\ftilde}\,.
	\end{align}
	By applying \Cref{Thm:Rob_Aux_Representer}, we can expand both of the differences on the right hand side as
	\begin{align}\label{eq:Rob_Aux_Kern_DifferenzAufteilung}
		\fj-\ftilde =&\ -\frac{1}{2\lb}\cdot\int h_i(x,y)\Phi_i(x)\,d\P(x,y) + \frac{1}{2\lb}\cdot\int\tilde{h}(x,y)\tilde{\Phi}(x)\,d\P(x,y) \notag\\
		=&\ \frac{1}{2\lb}\cdot\int h_i(x,y)\left(\tilde{\Phi}(x)-\Phi_i(x)\right)\,d\P(x,y) + \frac{1}{2\lb}\cdot\int\left(\tilde{h}(x,y)-h_i(x,y)\right)\tilde{\Phi}(x)\,d\P(x,y)
	\end{align}
	with $h_i$ and $\tilde{h}$ from the subdifferential of \lossshift (with respect to \fj respectively \ftilde). Thus, \eqref{eq:Rob_InfNormHNorm} yields for $i=1,2$
	\begin{align}\label{eq:Rob_Aux_Kern_InfNormAufteilung}
		\normSup{\fj-\ftilde} \le&\ \normSup{\frac{1}{2\lb}\cdot\int h_i(x,y)\left(\tilde{\Phi}(x)-\Phi_i(x)\right)\,d\P(x,y)} \notag\\
		&\ + \normSup{\frac{1}{2\lb}\cdot\int\left(\tilde{h}(x,y)-h_i(x,y)\right)\tilde{\Phi}(x)\,d\P(x,y)} \notag\\
		\le&\ \normSup{\frac{1}{2\lb}\cdot\int h_i(x,y)\left(\tilde{\Phi}(x)-\Phi_i(x)\right)\,d\P(x,y)} \notag\\
		&\ + \normSup{\ktilde} \cdot \norm{\Htilde}{\frac{1}{2\lb}\cdot\int\left(\tilde{h}(x,y)-h_i(x,y)\right)\tilde{\Phi}(x)\,d\P(x,y)} \,.
	\end{align}
	Now, we can easily bound the first summand on the right hand side of \eqref{eq:Rob_Aux_Kern_InfNormAufteilung} by
	\begin{align}\label{eq:Rob_Aux_Kern_InfNormAbsch}
		\normSup{\frac{1}{2\lb}\cdot\int h_i(x,y)\left(\tilde{\Phi}(x)-\Phi_i(x)\right)\,d\P(x,y)} &\le \frac{1}{2\lb} \cdot \normSup{h_i} \cdot \sup_{x\in\X}\normSup{\tilde{\Phi}(x)-\Phi_i(x)}\notag\\
		&\le \frac{|\loss|_1}{2\lb} \cdot \normSup{\ktilde-\k_i}\,,
	\end{align}
	where we applied \Cref{Lem:Rob_Aux_NormIntegralTauschen} in the first step and obtained the bound for $h_i$ from \Cref{Thm:Rob_Aux_Representer}. 
	
	As for the square of the \Htilde-norm in the second summand on the right hand side of \eqref{eq:Rob_Aux_Kern_InfNormAufteilung}, applying \eqref{eq:Rob_Aux_Kern_DifferenzAufteilung} yields
	\begin{align}\label{eq:Rob_Aux_Kern_HNormAufteilung}
		&\norm{\Htilde}{\frac{1}{2\lb}\cdot\int\left(\tilde{h}(x,y)-h_i(x,y)\right)\tilde{\Phi}(x)\,d\P(x,y)}^2 \notag\\
		&= \left\langle \frac{1}{2\lb}\cdot\int\left(\tilde{h}(x,y)-h_i(x,y)\right)\tilde{\Phi}(x)\,d\P(x,y)\,, \fj-\ftilde\right\rangle_{\Htilde} \notag\\
		&\hspace*{0.3cm} - \left\langle \frac{1}{2\lb}\cdot\int\left(\tilde{h}(x,y)-h_i(x,y)\right)\tilde{\Phi}(x)\,d\P(x,y)\,,\frac{1}{2\lb}\cdot\int h_i(x',y')\left(\tilde{\Phi}(x')-\Phi_i(x')\right)\,d\P(x',y')  \right\rangle_{\Htilde}\,,
	\end{align}
	where we can apply the reproducing property \eqref{eq:Rob_ReprodProperty} to the first of these two inner products in order to obtain
	\begin{align*}
		&\left\langle \frac{1}{2\lb}\cdot\int\left(\tilde{h}(x,y)-h_i(x,y)\right)\tilde{\Phi}(x)\,d\P(x,y),\fj-\ftilde \right\rangle_{\Htilde}\\
		&= \frac{1}{2\lb}\cdot\int\left(\tilde{h}(x,y)-h_i(x,y)\right)\left(\fj(x)-\ftilde(x)\right)\,d\P(x,y) \le 0\,.
	\end{align*}
	This inequality holds true because \lossshift is convex which implies that for all $(x,y)\in\XY$ we have $s_1\le s_2$ for every $s_1\in \partial \lossshift(x,y,t_1)$, $s_2\in \partial\lossshift(x,y,t_2)$ with $t_1\le t_2$. Now there are two cases: Either at least one of the two factors in the integrand is zero or the two factors have different signs. Therefore, the integrand, and hence also the whole integral, is non-positive.
	
	Plugging this result into \eqref{eq:Rob_Aux_Kern_HNormAufteilung} results in
	\begin{align}\label{eq:Rob_Aux_Kern_HNormAbsch}
		&\norm{\Htilde}{\frac{1}{2\lb}\cdot\int\left(\tilde{h}(x,y)-h_i(x,y)\right)\tilde{\Phi}(x)\,d\P(x,y)}^2 \notag\\
		&\le \left|\left\langle \frac{1}{2\lb}\cdot\int\left(\tilde{h}(x,y)-h_i(x,y)\right)\tilde{\Phi}(x)\,d\P(x,y)\,,\frac{1}{2\lb}\cdot\int h_i(x',y')\left(\tilde{\Phi}(x')-\Phi_i(x')\right)\,d\P(x',y')  \right\rangle_{\Htilde}\right|\notag\\
		&= \frac{1}{4\lb^2} \cdot \left|\int\int \left(\tilde{h}(x,y)-h_i(x,y)\right) h_i(x',y')\left(\ktilde(x,x')-\k_i(x,x') \right) \,d\P(x',y') \,d\P(x,y)\right|\notag\\
		&\le \frac{1}{4\lb^2} \cdot \normSup{\tilde{h}-h_i} \cdot \normSup{h_i} \cdot \normSup{\ktilde-\k_i}\notag\\
		&\le \frac{|\loss|_1^2}{2\lb^2} \cdot \normSup{\ktilde-\k_i}\,,
	\end{align}
	where we again applied the reproducing property in the second step and \Cref{Thm:Rob_Aux_Representer} for bounding $\tilde{h}-h_i$ and $h_i$ in the last step.
	
	By the definition of \ktilde, we further know that
	\begin{align*}
		\normSup{\ktilde-\k_i} = \normSup{\frac{\k_1+\k_2}{2}-\k_i} = \normSup{\frac{\k_1-\k_2}{2}} = \frac{\normSup{\k_1-\k_2}}{2}
	\end{align*}
	for $i=1,2$, as well as
	\begin{align*}
		\normSup{\ktilde} = \normSup{\frac{\k_1+\k_2}{2}} \le \frac{\normSup{\k_1}+\normSup{\k_2}}{2} \le \max\left\{\normSup{\k_1},\normSup{\k_2}\right\} = \kappa\,.
	\end{align*}
	Thus, we obtain the assertion by combining \eqref{eq:Rob_Aux_Kern_Triangle} with \eqref{eq:Rob_Aux_Kern_InfNormAufteilung}, \eqref{eq:Rob_Aux_Kern_InfNormAbsch} and \eqref{eq:Rob_Aux_Kern_HNormAbsch}:
	\begin{align*}
		\normSup{\fa-\fb} &\le \normSup{\fa-\ftilde} + \normSup{\fb-\ftilde}\\
		&\le \sum_{i=1}^{2} \left(\frac{|\loss|_1}{2\lb} \cdot \normSup{\ktilde-\k_i} + \normSup{\ktilde} \cdot \frac{|\loss|_1}{\sqrt{2}\lb}\cdot \sqrt{\normSup{\ktilde-\k_i}}  \right)\\
		&\le \frac{|\loss|_1}{\lb} \cdot \left(\frac{1}{2} \cdot \normSup{\k_1-\k_2} + \kappa \cdot \sqrt{\normSup{\k_1-\k_2}}\right)\,.\qedhere
	\end{align*}
\end{proof}

We can now progress to the analogous decomposition of $\normLppix{\froba-\frobb}$. However, we only need to prove an analogous result to \Cref{Lem:Rob_Aux_Kern} but not to \Cref{Lem:Rob_Aux_P,Lem:Rob_Aux_lambda}, since our analysis of the first two summands on the right hand side of \eqref{eq:Rob_Aux_DecompositionLp} in the proof of \Cref{Thm:Rob_StabilityLp} will be based directly on \Cref{Lem:Rob_Aux_P} respectively \Cref{Lem:Rob_Aux_lambda} and the fact that
\begin{align}\label{eq:Rob_Aux_NormLpNormInf}
	\normLppx{g}\le\normSup{g}
\end{align}
for all bounded functions $g$ and all $p\in[1,\infty)$.

\begin{mylem}\label{Lem:Rob_Aux_KernLp}
	Let \X be a complete and separable metric space and $\Y\subseteq\R$ be closed. Let $\P\in\MXY$ be a probability measure, $\lb>0$ and $\k_1,\k_2$ be bounded and measurable kernels on \X with separable RKHSs $\H_1,\H_2$. Denote $\kappa:=\max\{\normSup{\k_1},\normSup{\k_2}\}$. Let \loss be a convex and Lipschitz continuous loss function and let $p\in[1,\infty)$. Then,
	\begin{align*}
		\normLppx{f_{\P,\lb,\k_1}-f_{\P,\lb,\k_2}} \le  \frac{|\loss|_1}{\lb} \cdot \left(\frac{1}{2} \cdot \normLppxb{\k_1-\k_2} + \kappa \cdot \sqrt{\normLppxb{\k_1-\k_2}}\right)\,.
	\end{align*}
\end{mylem}

\begin{proof}
	The proof is almost identical to that of \Cref{Lem:Rob_Aux_Kern} with $\normSup{\cdot}$ being replaced by $\normLppx{\cdot}$, for which reason we will only highlight the differences here.
	
	First of all, because of \eqref{eq:Rob_Aux_NormLpNormInf}, we obtain analogously to \eqref{eq:Rob_Aux_Kern_InfNormAufteilung}
	\begin{align*}
		\normLppx{f_i-\ftilde} \le&\ \normLppx{\frac{1}{2\lb}\cdot\int h_i(x,y)\left(\tilde{\Phi}(x)-\Phi_i(x)\right)\,d\P(x,y)} \notag\\
		&\ + \normSup{\frac{1}{2\lb}\cdot\int\left(\tilde{h}(x,y)-h_i(x,y)\right)\tilde{\Phi}(x)\,d\P(x,y)} \notag\\
		\le&\ \normLppx{\frac{1}{2\lb}\cdot\int h_i(x,y)\left(\tilde{\Phi}(x)-\Phi_i(x)\right)\,d\P(x,y)} \notag\\
		&\ + \normSup{\ktilde} \cdot \norm{\Htilde}{\frac{1}{2\lb}\cdot\int\left(\tilde{h}(x,y)-h_i(x,y)\right)\tilde{\Phi}(x)\,d\P(x,y)} \,.
	\end{align*}
	Then, the first summand on the right hand side can be bounded in an analogous way to \eqref{eq:Rob_Aux_Kern_InfNormAbsch}:
	\begin{align*}
		&\normLppx{\frac{1}{2\lb}\cdot\int_{\XY}h_i(x,y)\left(\tilde{\Phi}(x)-\Phi_i(x)\right)\,d\P(x,y)}\\ 
		&\le \frac{1}{2\lb} \int_{\XY} \normLppx{h_i(x,y)\left(\tilde{\Phi}(x)-\Phi_i(x)\right)}\,d\P(x,y)\\
		&\le \frac{1}{2\lb} \cdot \normSup{h_i} \cdot \int_{\X}\normLppx{\tilde{\Phi}(x)-\Phi_i(x)}\,d\Px(x)\notag\\
		&= \frac{1}{2\lb} \cdot \normSup{h_i} \cdot \int_{\X}\left(\int_{\X}\left|\ktilde(x,x')-\k_i(x,x')\right|^p \, d\Px(x) \right)^{1/p} \,d\Px(x)\\
		&\le \frac{|\loss|_1}{2\lb} \cdot \normLppxb{\ktilde-\k_i}\,,
	\end{align*} 
	where we applied \Cref{Lem:Rob_Aux_NormIntegralTauschen} in the first step, and \Cref{Thm:Rob_Aux_Representer} (for obtaining the bound on $h_i$) as well as H\"older's inequality in the last step.
	Finally, we can tighten the bound from the last steps of \eqref{eq:Rob_Aux_Kern_HNormAbsch} in the following way:
	\begin{align*}
		&\frac{1}{4\lb^2} \cdot \left|\int\int \left(\tilde{h}(x,y)-h_i(x,y)\right) h_i(x',y')\left(\ktilde(x,x')-\k_i(x,x') \right) \,d\P(x',y') \,d\P(x,y)\right|\notag\\
		&\le \frac{|\loss|_1^2}{2\lb^2} \cdot \int\int\left| \ktilde(x,x')-\k_i(x,x') \right|\,d\P(x',y')\,d\P(x,y)\notag\\
		&= \frac{|\loss|_1^2}{2\lb^2} \cdot \normLapxb{\ktilde-\k_i}\notag\\
		&\le \frac{|\loss|_1^2}{2\lb^2} \cdot \normLppxb{\ktilde-\k_i}\,.
	\end{align*}
	The assertion then follows in the same way as in the proof of \Cref{Lem:Rob_Aux_Kern}.
\end{proof}

%% file: Sections/B02_Stability.tex
\subsection{Proofs for Section \ref{Sec:Rob_Stab}}\label{Sec:Rob_Proof_Stab}

\begin{proof}[Proof of \Cref{Thm:Rob_Stability}]
	Applying \Cref{Lem:Rob_Aux_P,Lem:Rob_Aux_lambda,Lem:Rob_Aux_Kern} to the decomposition \eqref{eq:Rob_Aux_Decomposition} of $\normSup{\froba-\frobb}$ yields
	\begin{align*}
		\normSup{\froba-\frobb} \le&\ \frac{\normSup{\k_1}^2 |\loss|_1}{\lb_1}\cdot \norm{tv}{\P_1-\P_2}  +  \frac{\normSup{\k_1}^2 |\loss|_1}{\min\{\lb_1,\lb_2\}^2}\cdot |\lb_1-\lb_2|\\
		&\ + \frac{|\loss|_1}{\lb_2} \cdot \left(\frac{1}{2} \cdot \normSup{\k_1-\k_2} + \kappa \cdot \sqrt{\normSup{\k_1-\k_2}}\right)\,.
	\end{align*}
	Since the order of decomposition can of course be freely varied, we also obtain analogous bounds with $\k_1$ being replaced by $\k_2$ (and vice versa) as well as $\lb_1$ by $\lb_2$ (and vice versa) in some of these summands. Since the right hand side of the assertion is greater or equal to the right hand sides of all of the bounds generated this way, the assertion directly follows.
\end{proof}

\begin{proof}[Proof of \Cref{Thm:Rob_StabilityLp}]
	Applying \Cref{Lem:Rob_Aux_KernLp} as well as \Cref{Lem:Rob_Aux_P,Lem:Rob_Aux_lambda} in combination with \eqref{eq:Rob_Aux_NormLpNormInf} to the decomposition \eqref{eq:Rob_Aux_DecompositionLp} of $\normLppix{\froba-\frobb}$ yields for $i=2$
	\begin{align*}
		&\norm{\Lppix[2]}{\froba-\frobb}\\ 
		&\le \frac{\normSup{\k_1}^2 |\loss|_1}{\lb_1}\cdot \norm{tv}{\P_1-\P_2}  +  \frac{\normSup{\k_1}^2 |\loss|_1}{\min\{\lb_1,\lb_2\}^2}\cdot |\lb_1-\lb_2|\\
		&\hspace*{0.5cm} + \frac{|\loss|_1}{\lb_2} \cdot \left(\frac{1}{2} \cdot \norm{\Lppixb[2]}{\k_1-\k_2} + \kappa \cdot \sqrt{\norm{\Lppixb[2]}{\k_1-\k_2}}\right)\\
		&\le \frac{\kappa^2 |\loss|_1}{\tau}\cdot \norm{tv}{\P_1-\P_2}  +  \frac{\kappa^2 |\loss|_1}{\tau^2}\cdot |\lb_1-\lb_2|\\ 
		&\hspace*{0.5cm}+ \frac{|\loss|_1}{\tau} \cdot \left(\frac{1}{2} \cdot \norm{\Lppixb[2]}{\k_1-\k_2} + \kappa \cdot \sqrt{\norm{\Lppixb[2]}{\k_1-\k_2}}\right)\,.
	\end{align*}
	Analogously, reversing the order of decomposition (such that \Cref{Lem:Rob_Aux_KernLp} can be applied to a summand with probability measure $\P_1$ in both SVMs) yields for $i=1$
	\begin{align*}
		&\norm{\Lppix[1]}{\froba-\frobb}\\ 
		&\le \frac{|\loss|_1}{\lb_1} \cdot \left(\frac{1}{2} \cdot \norm{\Lppixb[1]}{\k_1-\k_2} + \kappa \cdot \sqrt{\norm{\Lppixb[1]}{\k_1-\k_2}}\right)\\
		&\hspace*{0.5cm}  +  \frac{\normSup{\k_2}^2 |\loss|_1}{\min\{\lb_1,\lb_2\}^2}\cdot |\lb_1-\lb_2| + \frac{\normSup{\k_2}^2 |\loss|_1}{\lb_2}\cdot \norm{tv}{\P_1-\P_2}\\
		&\le \frac{\kappa^2 |\loss|_1}{\tau}\cdot \norm{tv}{\P_1-\P_2}  +  \frac{\kappa^2 |\loss|_1}{\tau^2}\cdot |\lb_1-\lb_2|\\
		&\hspace*{0.5cm} + \frac{|\loss|_1}{\tau} \cdot \left(\frac{1}{2} \cdot \norm{\Lppixb[1]}{\k_1-\k_2} + \kappa \cdot \sqrt{\norm{\Lppixb[1]}{\k_1-\k_2}}\right)\,.\qedhere
	\end{align*}
\end{proof}

%% file: Sections/B03_Localization.tex
\subsection{Proofs for Section \ref{Sec:Rob_Loc}}\label{Sec:Rob_Proof_Loc}

\begin{proof}[Proof of \Cref{Thm:Rob_Loc_GleicheAufteilung}]
	To shorten the notation, we define $\fj:=\fiext$ and $\fib:=\fibext$, $i=1,2$, $b=1,\dots,B$, in this proof. By the definition of \fa and \fb we know that
	\begin{align}\label{eq:Rob_Loc_ThmGleicheAufteilung_Absch}
		\normSup{\fa-\fb} &\le \sup_{x\in\X}\, \sum_{b=1}^B w_b(x) \cdot \left|\fibdach[1](x)-\fibdach[2](x) \right|\notag\\
		&\le \sup_{x\in\X}\, \max_{b\in\{1,\dots,B\}} \left|\fibdach[1](x)-\fibdach[2](x) \right|\notag\\
		&= \max_{b\in\{1,\dots,B\}} \normSup{\fibdach[1]-\fibdach[2]}\,,
	\end{align}
	where we applied \textbf{(W1)} and \textbf{(W2)} in the second step. 
	Since the functions \fibdach have not been defined as SVMs but instead as zero-extensions of SVMs \fibext on \Xb, we cannot apply \Cref{Thm:Rob_Stability} to the right hand side of \eqref{eq:Rob_Loc_ThmGleicheAufteilung_Absch} yet. However, these functions can actually be seen as SVMs on \X themselves, $\fibdach = \fibdachext$:
	
	According to \citet[Lemma 2]{meister2016}, we have $\Hibdach=\left\{\left.\hat{g}\,\right|\,g\in\Hib\right\}$ and $\norm{\Hibdach}{\hat{g}}=\norm{\Hib}{g}$ for all $g\in\Hib$. Since additionally $\riskshift[\lossshift,\Pibdach](\hat{g})=\riskshift[\lossshift,\Pib](g)$ for all $g\in\Hib$ (because the whole probability mass of \Pibdach is on \Xb where $\hat{g}$ and $g$ coincide), \eqref{eq:Rob_DefShiftSVM} yields $\fibdachext = \hat{f}_{\Pib,\lbib,\kib}\, (=\fibdach)$.
	
	Thus, we can apply \Cref{Thm:Rob_Stability} to the right hand side of \eqref{eq:Rob_Loc_ThmGleicheAufteilung_Absch} since the functions \fibdach are actually SVMs on the complete space \X (whereas the functions \fib are SVMs on the not necessarily complete spaces \Xb for which reason the theorem can not be applied to $\normSup{\fib[1]-\fib[2]}$ even though this term is obviously equivalent to $||\fibdach[1]-\fibdach[2]||_\infty$). By doing this, the first assertion follows, but with every \Pib replaced by \Pibdach and \kib by \kibdach. Because of them just being zero-extensions of \Pib and \kib respectively however, this does not influence the respective norms.
\end{proof}

\begin{proof}[Proof of \Cref{Thm:Rob_Loc_GleicheAufteilungLp}]
	To shorten the notation, we define $\fj:=\fiext$ and $\fib:=\fibext$, $i=1,2$, $b=1,\dots,B$, in this proof. By the definition of \fa and \fb we know that
	\begin{align}\label{eq:Rob_Loc_ThmGleicheAufteilungLp_Absch}
		\normLppix{\fa-\fb} &\le \sum_{b=1}^B \normLppix{w_b\cdot \left(\fibdach[1]-\fibdach[2]\right)}\notag\\
		&\le \sum_{b=1}^B \left(\int_{\X} \left| \fibdach[1](x)-\fibdach[2](x) \right|^p\,d\Pix(x) \right)^{1/p}\notag\\
		&= \sum_{b=1}^B \left( \Pix(\Xb) \cdot \int_{\Xb} \left| \fibdach[1](x)-\fibdach[2](x) \right|^p\,d\Pibx(x) \right)^{1/p}\notag\\
		&= \sum_{b=1}^B  \left(\Pix(\Xb)\right)^{1/p} \cdot \left( \int_{\X} \left| \fibdach[1](x)-\fibdach[2](x) \right|^p\,d\Pibdachx(x) \right)^{1/p}\notag\\
		&= \sum_{b=1}^B  \left(\Pix(\Xb)\right)^{1/p} \cdot \normLppibdachx{\fibdach[1]-\fibdach[2]}\,.
	\end{align}
	Here, we applied \textbf{(W1)} in the second, \fibdach[1] and \fibdach[2] being zero on $\X\setminus\Xb$ in combination with \eqref{eq:Rob_Loc_Pib} in the third, and the definition of \Pibdach as zero-extension of \Pib in the fourth step.
	
	Noting that \fibdach[1] and \fibdach[2] are SVMs on \X themselves, $\fibdach=\fibdachext$ (\vgl proof of \Cref{Thm:Rob_Loc_GleicheAufteilung}), we can now apply \Cref{Thm:Rob_StabilityLp} to the norms on the right hand side of \eqref{eq:Rob_Loc_ThmGleicheAufteilungLp_Absch}. This yields the assertion (as in the proof of \Cref{Thm:Rob_Loc_GleicheAufteilung} with $\Pibdach$ and $\kibdach$ instead ob \Pib and \kib which does not change the respective norms).
\end{proof}

\begin{proof}[Proof of \Cref{Thm:Rob_Loc_VerschAuftLp}]
	In addition to the auxiliary distributions and kernels introduced prior to \Cref{Thm:Rob_Loc_VerschAuftLp}, we also need auxiliary regularization parameters in this proof. We denote these parameters by $\lbijbstar:=(\Pjaibx{j}(\Xbstar))^{-1}\lbiaib$ for $i,j=1,2$ and $b=1,\dots,B$.
	
	By applying the triangle inequality we can now expand the norm we have to investigate as
	\begin{align}\label{eq:Rob_Loc_ThmVerschAuftLp_Triangle}
		\normLapix[1]{\fiextreg[1]-\fiextreg[2]} \le&\ \normLapix[1]{\fiextreg[1]-\fastarproofLp}\notag\\ 
		&\ +  \normLapix[1]{\fastarproofLp-\fbstarproofLp}\notag\\
		&\ + \normLapix[1]{\fbstarproofLp-\fiextreg[2]}
	\end{align}
	with $\lbijstarbold:=(\lbijbstar[i,j,1],\dots,\lbijbstar[i,j,B])$ and $\kistarbold:=(\kibstar[i,1],\dots,\kibstar[i,B])$ for $i,j=1,2$, and the newly introduced localized SVMs being defined as in \eqref{eq:Rob_Loc_DefGlobalPredictorDiffLoc}. We will now examine the three norms from the right hand side of \eqref{eq:Rob_Loc_ThmVerschAuftLp_Triangle} separately:
	
	\begin{enumerate}[label=(\roman*)]
		\item For a function $g:\Xbstar\to\R$, denote by $\tilde{g}$ its zero-extension to \Xiaib[1] (respectively to $\Xiaib[1]\times\Xiaib[1]$ if the function is instead defined on $\Xbstar\times\Xbstar$). Defining $\kia[1]^\circ := \sum_{b\in J_{1,a}} \tilde{\k}_{1,b}^*$ yields for all $a\in\{1,\dots,A_1\}$ new local SVMs $f_{\Pia[1],\lbia[1],\kia[1]^\circ}$ which by \eqref{eq:Rob_DefShiftSVM} are defined as
		\begin{align}\label{eq:Rob_Loc_ThmVerschAuftLp_TeilIDefSVM}
			f_{\Pia[1],\lbia[1],\kia[1]^\circ} = \arg\inf_{f\in\Hia[1]^\circ} \riskshift[\lossshift,\P_{1,a}](f) + \lbia[1] \norm{\Hia[1]^\circ}{f}^2\,.
		\end{align}
	 	Now, combining Theorem 5 from \cite{berlinet2004} and Lemma 2 from \cite{meister2016} yields that 
	 	\begin{align*}
	 		\Hia[1]^\circ = \left\{f:\Xia[1]\to\R\,\left|\, f=\sum_{b\in J_{1,a}}\tilde{f}_b,\, f_b\in\Hib[1]^* \text{ for } b=1,\dots,B\right.\right\}\,,
	 	\end{align*}
 		with the decomposition of each such $f\in\Hia[1]^\circ$ being unique because of the sets $\Xbstar[1],\dots,\Xbstar[B]$, the domains of the functions $f_b$, being pairwise disjoint, \vgl \textbf{(R1')} and \textbf{(R2')}. Thus, the mentioned results also yield $\norm{\Hia[1]^\circ}{f}^2 = \sum_{b\in J_{1,a}} \norm{\Hib[1]^*}{f_b}^2$ for all $f\in\Hia[1]^\circ$. Additionally, again because of the domains of the functions $f_b$ being pairwise disjoint, we are able to also expand the risk from  \eqref{eq:Rob_Loc_ThmVerschAuftLp_TeilIDefSVM} similarly to the preceding expansion of the $\Hia[1]^\circ$-norm:
		\begin{align*}
			\riskshift[\lossshift,\P_{1,a}](f) &= \int_{\Xia[1]} \lossshift(x,y,f(x))\,d\Pia[1](x,y)\\ 
			&= \sum_{b\in J_{1,a}} \int_{\Xbstar} \lossshift(x,y,f_b(x))\,d\Pia[1](x,y)\\
			&= \sum_{b\in J_{1,a}} \Piax[1](\Xbstar)\cdot \int_{\Xbstar} \lossshift(x,y,f_b(x))\,d\Pibstar[1](x,y)\\
			&= \sum_{b\in J_{1,a}} \Piax[1](\Xbstar)\cdot \riskshift[\lossshift,\P_{1,b}^*](f_b)\,,
		\end{align*}
		where we applied \eqref{eq:Rob_Loc_Pib} in the third step.
		
		By plugging this into  \eqref{eq:Rob_Loc_ThmVerschAuftLp_TeilIDefSVM}, we obtain
		\begin{align*}
			f_{\Pia[1],\lbia[1],\kia[1]^\circ} &= \arg\inf_{f\in\Hia[1]^\circ} \sum_{b\in J_{1,a}} \left( \Piax[1](\Xbstar)\cdot \riskshift[\lossshift,\P_{1,b}^*](f_b) + \lbia[1]  \norm{\Hib[1]^*}{f_b}^2 \right)\\
			&= \sum_{b\in J_{1,a}} \widetilde{\arg\inf_{{f}_b\in\Hib[1]^*}} \left( \Piax[1](\Xbstar)\cdot \riskshift[\lossshift,\P_{1,b}^*](f_b) + \lbia[1]  \norm{\Hib[1]^*}{f_b}^2 \right)\\
			&= \sum_{b\in J_{1,a}} \widetilde{\arg\inf_{{f}_b\in\Hib[1]^*}} \left( \riskshift[\lossshift,\P_{1,b}^*](f_b) + \frac{\lbia[1]}{\Piax[1](\Xbstar)}  \norm{\Hib[1]^*}{f_b}^2 \right)\\
			&= \sum_{b\in J_{1,a}} \fabstartildeproofLp
		\end{align*}
		and thus
		\begin{align*}
			\fastarproofLp = \sum_{b=1}^B \fabstarhatproofLp = \sum_{a=1}^{A_1}\sum_{b\in J_{1,a}} \fabstarhatproofLp = \sum_{a=1}^{A_1} \hat{f}_{\Pia[1],\lbia[1],\kia[1]^\circ}\,.
		\end{align*}
		We can therefore also interpret the first difference on the right hand side of \eqref{eq:Rob_Loc_ThmVerschAuftLp_Triangle} as the difference between two localized SVMs that are based on the same regionalization \XiAibold[1] (and on the same probability measure and vector of regularization parameters). An application of \Cref{Thm:Rob_Loc_GleicheAufteilungLp} hence yields
		\begin{align*}
			&\normLapix[1]{\fiextreg[1]-\fastarproofLp} = \normLapix[1]{\fiextreg[1]-\sum_{a=1}^{A_1} \hat{f}_{\Pia[1],\lbia[1],\kia[1]^\circ}}\\
			&\le |\loss|_1 \cdot \sum_{a=1}^{A_1} \Pix[1](\Xia[1]) \cdot \Bigg( \frac{1}{2\lbia[1]} \cdot \normLapiaxb[1]{\kia[1]-\kia[1]^\circ}\\ 
			&\hspace*{4cm} + \frac{\max\left\{\normSup{\kia[1]},\,\normSup{\kia[1]^\circ}\right\}}{\lbia[1]} \cdot \sqrt{\normLapiaxb[1]{\kia[1]-\kia[1]^\circ}} \Bigg)\,.
		\end{align*}
		Because
		\begin{align*}
			\kia[1]^\circ(x,x') = \begin{cases}
				\kia[1](x,x')  &\wenn \exists\, b\in J_{1,a}: x,x'\in\Xbstar\,,\\
				0 & \sonst\,,
			\end{cases}
		\end{align*}
		we furthermore know that $\max\left\{\normSup{\kia[1]},\,\normSup{\kia[1]^\circ}\right\} = \normSup{\kia[1]}$ and
		\begin{align*}
			\normLapiaxb[1]{\kia[1]-\kia[1]^\circ} &= \int_{\Xia[1]}\int_{\Xia[1]} \left|\kia[1](x,x')-\kia[1]^\circ(x,x')\right| \, d\Piax[1](x')d\Piax[1](x)\\
			&= \sum_{b\in J_{1,a}} \int_{\Xbstar} \int_{\Xia[1]\setminus\Xbstar} \left| \kia[1](x,x') \right| \,d\Piax[1](x')d\Piax[1](x)\\
			&\le \normSup{\kia[1]}^2 \cdot \sum_{b\in J_{1,a}} \Piax[1](\Xbstar) \cdot \left(1-\Piax[1](\Xbstar)\right)
		\end{align*}
		which finally results in
		\begin{align*}
			&\normLapix[1]{\fiextreg[1]-\fastarproofLp}\\
			&\le |\loss|_1 \cdot \sum_{a=1}^{A_1} \Pix[1](\Xia[1]) \cdot \Bigg( \frac{\normSup{\kia[1]}^2}{2\lbia[1]} \cdot \sum_{b\in J_{1,a}} \Piax[1](\Xbstar) \cdot \left(1-\Piax[1](\Xbstar)\right)\\ 
			&\hspace*{4cm} + \frac{\normSup{\kia[1]}^2}{\lbia[1]} \cdot \sqrt{\sum_{b\in J_{1,a}} \Piax[1](\Xbstar) \cdot \left(1-\Piax[1](\Xbstar)\right)} \Bigg)\,.
		\end{align*}
	
		\item The second norm on the right hand side of \eqref{eq:Rob_Loc_ThmVerschAuftLp_Triangle} already consists of the difference of two localized SVMs that are based on the same regionalization \XBstarbold (and the same probability measure), without us needing to make any changes before. We can therefore directly apply \Cref{Thm:Rob_Loc_GleicheAufteilungLp} and obtain
		\begin{align}\label{eq:Rob_Loc_ThmVerschAuftLp_TeilIIAbsch}
			&\normLapix[1]{\fastarproofLp-\fbstarproofLp}\notag\\
			&\le |\loss|_1 \cdot \sum_{b=1}^B \Pix[1](\Xbstar) \cdot \Bigg( 	\frac{(\kappa_b^*)^2}{(\tau_{1,b}^*)^2} \cdot \left|\lbijbstar[1,1,b]-\lbijbstar[2,1,b] \right| + \frac{1}{2\tau_{1,b}^*} \cdot \normLapibstarxb[1]{\kibstar[1,b]-\kibstar[2,b]} \notag\\
			&\hspace*{7cm} + \frac{\kappa_b^*}{\tau_{1,b}^*} \cdot 	\sqrt{\normLapibstarxb[1]{\kibstar[1,b]-\kibstar[2,b]}} \Bigg)
		\end{align}
		with
		\begin{align*}
			\kappa_b^* := 	\max\left\{\normSup{\kibstar[1,b]},\normSup{\kibstar[2,b]}\right\} \le \max\left\{\normSup{\kiaib[1]},\normSup{\kiaib[2]}\right\} = \kappa_b\,,
		\end{align*}
		because \kibstar and \kiaib coincide everywhere \kibstar is defined, and
		\begin{align*}
			\tau_{1,b}^* := \min\left\{\lbijbstar[1,1,b],\lbijbstar[2,1,b]\right\} \ge 	\min\left\{ \lbiaib[1] , \lbiaib[2] \right\} = \tau_b
		\end{align*}
		because of \lbijbstar[i,1,b] being defined as $(\Pjaibx{1}(\Xbstar))^{-1}\lbiaib$. Thus, \eqref{eq:Rob_Loc_ThmVerschAuftLp_TeilIIAbsch} still holds true after replacing $\kappa_b^*$ and $\tau_{1,b}^*$ by $\kappa_b$ and $\tau_b$. Additionally, applying the definition of \lbijbstar[i,1,b] again as well as the definition of \Pjaibx{1} from \eqref{eq:Rob_Loc_Pib}, yields
		\begin{align*}
			\left|\lbijbstar[1,1,b]-\lbijbstar[2,1,b] \right| &= \frac{1}{\Pix[1](\Xbstar)} \cdot \left|\lbiaib[1]\cdot \Pix[1](\Xiaib[1]) - \lbiaib[2] \cdot \Pix[1](\Xiaib[2]) \right|\\
			&\le \frac{1}{\Pix[1](\Xbstar)} \cdot \Big(\lbiaib[1]\cdot \left|\Pix[1](\Xiaib[1]) - \Pix[1](\Xiaib[2]) \right|\\ 
			&\hspace*{5cm}+ \Pix[1](\Xiaib[2]) \cdot \left|\lbiaib[1]-\lbiaib[2] \right| \Big)
		\end{align*}
		as well as analogously
		\begin{align*}
			\left|\lbijbstar[1,1,b]-\lbijbstar[2,1,b] \right| &\le \frac{1}{\Pix[1](\Xbstar)} \cdot \Big(\lbiaib[2]\cdot \left|\Pix[1](\Xiaib[1]) - \Pix[1](\Xiaib[2]) \right|\\ 
			&\hspace*{5cm}+ \Pix[1](\Xiaib[1]) \cdot \left|\lbiaib[1]-\lbiaib[2] \right| \Big)\,,
		\end{align*}
		and hence
		\begin{align*}
			&\left|\lbijbstar[1,1,b]-\lbijbstar[2,1,b] \right| \le \frac{1}{\Pix[1](\Xbstar)} \cdot \Big(\tau_b\cdot \left|\Pix[1](\Xiaib[1]) - \Pix[1](\Xiaib[2]) \right| + \rho_{1,b} \cdot \left|\lbiaib[1]-\lbiaib[2] \right| \Big)\,.
		\end{align*}
	
		Plugging this into \eqref{eq:Rob_Loc_ThmVerschAuftLp_TeilIIAbsch} finally yields
		\begin{align*}
			&\normLapix[1]{\fastarproofLp-\fbstarproofLp}\notag\\
			&\le |\loss|_1 \cdot \sum_{b=1}^B \Bigg( \rho_{1,b} \cdot \frac{\kappa_b^2}{\tau_b^2} \cdot \left|\lbiaib[1]-\lbiaib[2] \right|\\
			&\hspace*{2cm} + \Pix[1](\Xbstar) \cdot \bigg(\frac{1}{2\tau_b} \cdot \normLapibstarxb[1]{\kibstar[1,b]-\kibstar[2,b]}\\ 
			&\hspace*{5cm}+ \frac{\kappa_b}{\tau_b} \cdot \sqrt{\normLapibstarxb[1]{\kibstar[1,b]-\kibstar[2,b]}} \bigg)\\
			&\hspace*{2cm} + \frac{\kappa_b^2}{\tau_b} \cdot \left|\Pix[1](\Xiaib[1]) - \Pix[1](\Xiaib[2]) \right| \Bigg)\,.
		\end{align*}
	
		\item The third norm on the right hand side of \eqref{eq:Rob_Loc_ThmVerschAuftLp_Triangle} can be analyzed similarly to the first one. Let the $(\tilde{\cdot})$-notation now denote zero-extensions to \Xiaib[2] instead of \Xiaib[1]. Analogously to (i), it can be shown that 
		\begin{align*}
			\fbstarproofLp = \sum_{b=1}^B \fbbstarhatproofLp = \sum_{a=1}^{A_2}\sum_{b\in J_{2,a}} \fbbstarhatproofLp = \sum_{a=1}^{A_2} \hat{f}_{\Pjai[2]{1},\lbia[2],\kia[2]^\circ}\,,
		\end{align*}
		where $\kia[2]^\circ := \sum_{b\in J_{2,a}} \tilde{\k}_{2,b}^*$ for $a=1,\dots,A_2$. We can thus also interpret the third difference on the right hand side of \eqref{eq:Rob_Loc_ThmVerschAuftLp_Triangle} as the difference between two localized SVMs that are based on the same regionalization \XiAibold[2] (and on the same vector of regularization parameters) and apply \Cref{Thm:Rob_Loc_GleicheAufteilungLp}:
		\begin{align*}
			&\normLapix[1]{\fbstarproofLp-\fiextreg[2]} = \normLapix[1]{\sum_{a=1}^{A_2} \hat{f}_{\Pjai[2]{1},\lbia[2],\kia[2]^\circ} -\fiextreg[2]}\\
			&\le |\loss|_1 \cdot \sum_{a=1}^{A_2} \Pix[1](\Xia[2]) \cdot \Bigg(\frac{\normSup{\kia[2]}^2}{\lb_{2,a}} \cdot \norm{tv}{\Pjai[2]{1}-\Pjai[2]{2}} \\ &\hspace*{4cm}+\frac{\normSup{\kia[2]}^2}{2\lbia[2]} \cdot \sum_{b\in J_{2,a}} \Pjaix[2]{1}(\Xbstar) \cdot \left(1-\Pjaix[2]{1}(\Xbstar)\right)\\ 
			&\hspace*{4cm} + \frac{\normSup{\kia[2]}^2}{\lbia[2]} \cdot \sqrt{\sum_{b\in J_{2,a}} \Pjaix[2]{1}(\Xbstar) \cdot \left(1-\Pjaix[2]{1}(\Xbstar)\right)} \Bigg)\,,
		\end{align*}
		where we employed that $\max\left\{\normSup{\kia[2]},\,\normSup{\kia[2]^\circ}\right\}=\normSup{\kia[2]}$ and
		\begin{align*}
			\norm{\La(\Pjaix[2]{1}\otimes\Pjaix[2]{1})}{\kia[2]-\kia[2]^\circ} \le \normSup{\kia[2]}^2 \cdot \sum_{b\in J_{2,a}} \Pjaix[2]{1}(\Xbstar) \cdot \left(1-\Pjaix[2]{1}(\Xbstar)\right)\,,
		\end{align*}
		which follows in the same way as the analogous statements in (i).
	\end{enumerate}

	Plugging these three bounds into \eqref{eq:Rob_Loc_ThmVerschAuftLp_Triangle} and additionally observing
	\begin{align*}
		&\sum_{a=1}^{A_i} \Pix[1](\Xia) \cdot \Bigg(\frac{\normSup{\kia[i]}^2}{2\lbia[i]} \cdot \sum_{b\in J_{i,a}} \Pjaix{1}(\Xbstar) \cdot \left(1-\Pjaix{1}(\Xbstar)\right)\\ 
		&\hspace*{3.5cm} + \frac{\normSup{\kia}^2}{\lbia} \cdot \sqrt{\sum_{b\in J_{i,a}} \Pjaix{1}(\Xbstar) \cdot \left(1-\Pjaix{1}(\Xbstar)\right)} \Bigg)\\
		&\le\sum_{a=1}^{A_i}\sum_{b\in J_{i,a}} \Pix[1](\Xia) \cdot \frac{\normSup{\kia}^2}{\lbia}\cdot  \Bigg(\frac{1}{2} \cdot \Pjaix{1}(\Xbstar) \cdot \left(1-\Pjaix{1}(\Xbstar)\right)\\ 
		&\hspace*{6cm} + \sqrt{\Pjaix{1}(\Xbstar) \cdot \left(1-\Pjaix{1}(\Xbstar)\right)} \Bigg)\\
		&\le \sum_{b=1}^B \rho_{1,b} \cdot \frac{\kappa_b^2}{\tau_b} \cdot \left(\frac{\Pjaibx{1}(\Xbstar) \cdot \left(1-\Pjaibx{1}(\Xbstar)\right)}{2} + \sqrt{\Pjaibx{1}(\Xbstar) \cdot \left(1-\Pjaibx{1}(\Xbstar)\right)} \vphantom{\frac{\Pjaibx{1}(\Xbstar) \cdot \left(1-\Pjaibx{1}(\Xbstar)\right)}{2}} \right)\,,
	\end{align*}
	$i=1,2$, yields the assertion.
\end{proof}